\documentclass{article} % For LaTeX2e
\usepackage{iclr2027_conference,times}

\iclrfinalcopy

\usepackage{amsmath,amsfonts,bm}

\def\eqref#1{equation~\ref{#1}}
\def\1{\bm{1}}

\DeclareMathAlphabet{\mathsfit}{\encodingdefault}{\sfdefault}{m}{sl}
\SetMathAlphabet{\mathsfit}{bold}{\encodingdefault}{\sfdefault}{bx}{n}

\usepackage[linesnumbered,ruled,vlined]{algorithm2e}

\newtheorem{theorem}{Theorem}
\usepackage{amssymb}
\usepackage{subfig}
\usepackage[table]{xcolor}
\usepackage{hyperref} % 确保有超链接支持
\usepackage{amsmath}
\usepackage{multirow}
\SetKwInput{Input}{Input}
\SetKwInput{Output}{Output}
\SetKwComment{Comment}{$\triangleright$\ }{}
\usepackage[utf8]{inputenc} % allow utf-8 input
\usepackage[T1]{fontenc}    % use 8-bit T1 fonts
\usepackage{hyperref}       % hyperlinks
\usepackage{url}            % simple URL typesetting
\usepackage{booktabs}       % professional-quality tables
\usepackage{amsfonts}       % blackboard math symbols
\usepackage{nicefrac}       % compact symbols for 1/2, etc.
\usepackage{microtype}      % microtypography
\usepackage{xcolor}         % colors
\usepackage{graphicx}
\usepackage{subfig}
\ifodd 1
\else

\fi

\title{Redundancy Meets Synergy: Dependency-aware Expert Selection for MoE via Submodular Optimization
}

\author{%
  Zheng Lin$^{1}$, Shaoke Fang$^{2}$,  Yuxin Zhang$^{3}$, Jinfeng Xu$^{4}$, Zihan Fang$^{5}$, Zhe Chen$^{3}$,
   \\ \textbf{Wei Ni$^{6}$}, \textbf{Jun Luo$^{7}$}, \textbf{Symeon Chatzinotas$^{1}$} \\
   $^{1}$ Interdisciplinary Centre for Security, Reliability and Trust, University of Luxembourg, Luxembourg \\
  $^{2}$ Department of Computer Science, Peking University, China\\
  $^{3}$ Institute of Space Internet, Fudan University, China \\
  $^{4}$ Department of Electrical and Computer Engineering, The University of British Columbia, Canada \\
     $^{5}$ Department of Computer Science, City University of Hong Kong, Hong Kong, China \\
  $^{6}$ School of Engineering, Edith Cowan University, Australia \\
  $^{7}$ College of Computing and Data Science, Nanyang Technological University, Singapore\\
}

\begin{document}

\maketitle

\begin{abstract}
While Mixture-of-Experts (MoE) models effectively scale model capacity through sparse activation, their deployment is often bottlenecked by prohibitive memory requirements. Extracting a compact subset of experts presents a promising solution. However, existing expert selection heuristics predominantly rely on Top-$k$ ranking, which isolates the evaluation of individual experts and ignores the intricate inter-expert dependencies introduced by the MoE gating network. In this paper, we propose DS-MoE, a theoretically grounded framework that redefines expert selection via difference-of-submodular (DS) optimization. By analyzing the second-order Taylor expansion of the loss degradation, we reveal functional duality within expert combinations: redundancy (where experts encode overlapping representations) and synergy (where experts provide complementary error cancellation). To navigate this duality, we mathematically decouple redundancy reduction from synergy maximization by formulating the selection objective as a DS function. Furthermore, we devise a tailored majorization-minimization (MM) algorithm with provable guarantees to efficiently identify a high-quality expert subset. Extensive experiments demonstrate that DS-MoE effectively preserves indispensable expert combinations, achieving superior performance compared to the state-of-the-art baselines.
\end{abstract}

\section{Introduction}
\vspace{-0.7em}
The pursuit of scaling laws in large language models (LLMs)~\cite{kaplan2020scaling,hoffmann2022training,clark2022unified} has driven a fundamental paradigm shift from dense neural networks to Mixture-of-Experts (MoE) architectures~\cite{shazeer2017outrageously,liu2024survey,fedus2022switch}. In contrast to traditional training paradigms that activate all parameters for every input, the MoE architecture is built on the principle of conditional computation. It relies on a gating mechanism to dynamically route each input token to a sparse subset of experts (e.g., Top-$k$)~\cite{fedus2022switch}. This selective execution effectively decouples model capacity from computational cost, allowing MoE architectures to scale to the trillion-parameter regime~\cite{lepikhin2020gshard} and unlock massive knowledge capacity without a proportional increase in computing overhead~\cite{li2025uni,liu2024deepseek}.

Despite its significance, the theoretical understanding of expert contributions to model performance remains in its infancy. Most prior works evaluate experts in strict isolation~\cite{zhou2022mixture,muzio2024seer,baidiep}, assigning importance scores independently and thus ignoring intricate inter-expert dependencies~\cite{chi2022representation}, as detailed in Appendix~\ref{sec:related_work}. Theoretically, this strict independence assumption equates to approximating the loss landscape's Hessian matrix as purely diagonal.  Consequently, prior works fail to distinguish between functional redundancy (where experts encode overlapping representations) and synergy (where specific expert combinations provide complementary error cancellation). This oversight inevitably leads to sub-optimal expert selection: redundant experts monopolize capacity, while combinatorially indispensable but individually negligible synergistic components are discarded. Therefore, expert selection must transcend heuristic scalar sorting and be formulated as a rigorous optimization problem. However, transitioning from isolated scoring to collective evaluation inherently introduces an intractable combinatorial optimization problem. Constrained by the immense discrete space, prior works inevitably default to simplistic and retreat to independent scoring heuristics~\cite{zhou2022mixture,puigcerver2024sparse,lin2024scattermoe}. Thus, there is a critical need for a framework that can systematically optimize inter-expert synergy and mitigate redundancy without succumbing to combinatorial explosion.

\begin{figure}[t!]
  \centering
  % 建议使用 pdf 格式的矢量图，以保证在任何缩放倍数下都清晰
  \includegraphics[width=0.90\textwidth]{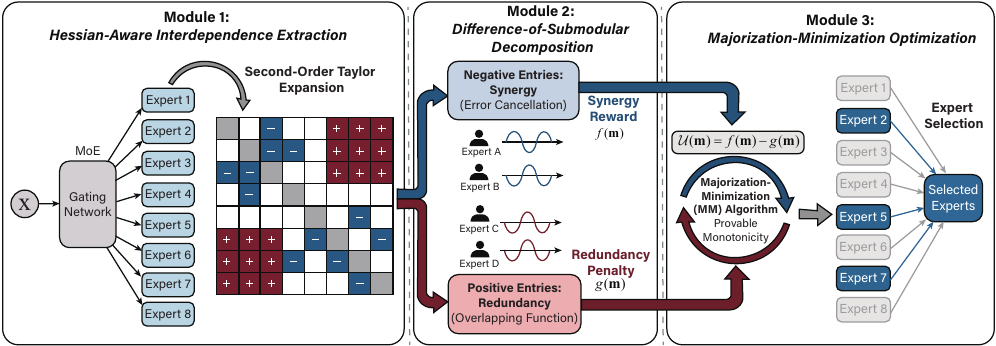}
  \caption{{Overview of the DS-MoE framework.} 
  }
  \label{fig:ds_moe_framework}
  \vspace{-1em}
\end{figure}

To bridge this gap, we propose DS-MoE, a theoretically grounded MoE framework that redefines expert selection via difference-of-submodular (DS) optimization~\cite{iyer2012algorithms}.
As depicted in Figure~\ref{fig:ds_moe_framework}, DS-MoE anchors its selection criterion in the second-order Taylor expansion of the loss degradation, effectively bridging the gap between continuous loss curvature and discrete expert selection~\cite{hassibi1992second}. This formulation reveals that the MoE Hessian encodes expert interdependence, manifesting fundamental duality: positive off-diagonal entries correspond to redundancy (penalizing overlapping functionality), while negative entries capture synergy (rewarding complementary error cancellation). To harness this duality, we decouple redundancy reduction from synergy maximization by formulating the selection objective as a DS optimization problem. Subsequently, we devise a tailored majorization-minimization (MM) algorithm~\cite{hunter2004tutorial,mairal2013stochastic} with guaranteed monotonic convergence to efficiently identify a high-quality expert combination.

Our contributions are summarized as follows: i) We derive a second-order characterization of inter-expert interdependence, where the off-diagonal entries of the MoE Hessian quantify the inter-expert functional redundancy and synergy; ii) We cast expert selection as a DS problem, thereby decomposing the collective impact of expert combinations into a controllable trade-off between penalizing redundancy and rewarding synergy; iii) We develop a tailored MM algorithm with guaranteed monotonic convergence to efficiently identify a high-quality expert subset; iv) Extensive experiments demonstrate that DS-MoE consistently outperforms state-of-the-art baselines, and strikes a superior accuracy-efficiency trade-off.

\vspace{-0.2em}
\section{Methodology} 
\vspace{-0.5em}

\subsection{Preliminaries \& Problem Formulation}\label{ssec:pro_formulated}
\vspace{-0.5em}

{\bf{Standard MoE mechanism.}} We begin by formalizing the MoE architecture~\citep{jacobs1991adaptive}. An MoE layer comprises a set of $N$ expert networks, denoted as $E = \{E_1, \dots, E_N\}$. For an input token $x \in \mathbb{R}^d$, a gating network maps $x$ to a sparse routing vector $G(x) \in \mathbb{R}^N$ via a learnable projection matrix $W_g \in \mathbb{R}^{N \times d}$, selecting the subset of Top-$K$ experts with the highest response values. The gating weights are derived by normalizing the projected logits over the selected subset:
\begin{equation}
G(x) = \operatorname{Softmax}(W_g x; \mathcal{T}(x)),
\end{equation}
where $\mathcal{T}(x) = \left\{ i \;\middle|\; \sum_{j=1}^N \mathbb{I}([W_g x]_j > [W_g x]_i) < K \right\}$ represents the set of selected expert indices, $\mathbb{I}(\cdot)$ denotes the indicator function, and $\operatorname{Softmax}(\cdot; \mathcal{T}(x))$ applies exponential normalization strictly to indices $i \in \mathcal{T}(x)$ while zeroing out others. The layer output $y$ is computed as the weighted sum of the active experts:
\begin{equation}
    y = \sum_{i \in \mathcal{T}(x)} G_{i}(x) \cdot E_{i}(x),
\end{equation}
where $G_i(x)$ is the $i$-th element of $G(x)$ that represents the routing score (i.e., routing probability) assigned to expert $E_i$, and $E_i(x)$ denotes the output representation generated by the expert network.

{\bf{Limitations of Isolated Expert Scoring.}} To retain a compact subset of experts, existing methods typically rank experts by isolated scores, such as routing frequency or gradient-based importance~\citep{muzio2024seer,baidiep}, and keep the Top-$k$ ones. This relies on a limiting independence assumption: the utility of an expert is evaluated without considering which other experts are retained. This greedy heuristic introduces a critical blind spot regarding inter-expert dependencies~\citep{chi2022representation}. Specifically, it fails to distinguish between functional redundancy and synergy. Consequently, it risks discarding experts that appear individually negligible but are combinatorially indispensable. By reducing expert selection to a series of independent scalar decisions rather than a combinatorial optimization, existing methods fail to exploit the complementary interactions required to drive loss reduction beyond the capacity of isolated experts.

{\bf{Formalizing Expert Selection.}}
To address these limitations and transition from heuristic scoring to a rigorous optimization framework, we introduce a binary selection vector $\mathbf{m} \in \{0,1\}^N$ for each MoE layer. Here, $m_i=1$ indicates that expert $E_i$ is retained, and $m_i=0$ indicates that it is removed. The selection is shared by all inputs, so removed experts need not be stored. Keeping the gating weights of the original model fixed, the output of the MoE layer is reformulated as:
\begin{equation}\label{eq:y}
    y(x; \mathbf{m}) = \sum_{i=1}^N m_i \cdot G_i(x) \cdot E_i(x),
\end{equation}
where $\mathbf{m}=\mathbf{1}$ recovers the full model with all $N$ experts retained.

We frame expert selection as a minimization problem. Our objective is to identify the sparse selection vector $\mathbf{m}$ that minimizes the expected loss degradation relative to the full model, subject to a cardinality constraint dictated by storage limitations (i.e., the model can accommodate $k$ experts, $\mathbf{1}^\top\mathbf{m} = k$):
\begin{equation}
\label{eq:problem}
    \min_{\mathbf{m} \in \{0, 1\}^N} \quad \underbrace{\mathbb{E}_{x \sim \mathcal{D}} \left[ \mathcal{L}(y(x; \mathbf{m})) - \mathcal{L}(y(x; \mathbf{1})) \right]}_{\text{Loss Degradation } \Delta \mathcal{L}} \quad \text{s.t.} \quad \mathbf{1}^\top\mathbf{m} = k.
\end{equation}
After selection, the retained experts and the router are fine-tuned on downstream tasks (Section~\ref{sec:experiments}).

{\bf{Optimization Challenges.}} Directly solving Eqn.~\ref{eq:problem} presents formidable obstacles. The cardinality constraint $\mathbf{1}^\top\mathbf{m} = k$ induces a combinatorial search space of size $\binom{N}{k}$, rendering the problem NP-hard~\cite{hartmanis1982computers}. In practice, brute-force forward propagation is computationally intractable due to the prohibitive latency of repeatedly retrieving expert parameters. To overcome these bottlenecks, we derive an efficient analytical proxy to quantify the collective importance of expert combinations by analyzing the second-order Taylor expansion~\citep{lecun1989optimal,dong2017learning} of the loss degradation in the following section.

\subsection{Hessian-Aware Interdependence Extraction}
\vspace{-0.1em}

{\bf{Second-Order Taylor Expansion of Loss Degradation.}} To circumvent the computational intractability of combinatorial search in Eqn.~\ref{eq:problem}, we reformulate the expert selection problem by analyzing the model's behavior around the full model ($\mathbf{m}=\mathbf{1}$). Selecting a sparse subset of experts can be modeled as a finite deviation from this reference point. By taking the second-order Taylor expansion of the loss degradation around $\mathbf{m}=\mathbf{1}$, we have:
\begin{equation}\label{eqn:loss_expansion}
\Delta\mathcal{L}(\mathbf{m}) = \mathcal{L}(\mathbf{m})-\mathcal{L}(\mathbf{1}) \approx \underbrace{(\mathbf{m}-\mathbf{1})^{\top} \left. \nabla_{\mathbf{m}}\mathcal{L} \right|_{\mathbf{m}=\mathbf{1}}}_{\text{First-order Term}} + \underbrace{\frac{1}{2}(\mathbf{m}-\mathbf{1})^{\top} \left. \mathbf{H}_{\mathbf{m}} \right|_{\mathbf{m}=\mathbf{1}} (\mathbf{m}-\mathbf{1})}_{\text{Second-order Term}},
\end{equation}
where $\nabla_{\mathbf{m}}\mathcal{L}$ and $\mathbf{H}_{\mathbf{m}} = \nabla^2_{\mathbf{m}}\mathcal{L}$ are the gradient vector and Hessian matrix of the loss function with respect to $\mathbf{m}$, respectively.

\begin{figure}[t!]
  \centering
  \subfloat[Gradients (First-Order)]{
    \includegraphics[width=0.256\textwidth,height=0.228\textwidth]{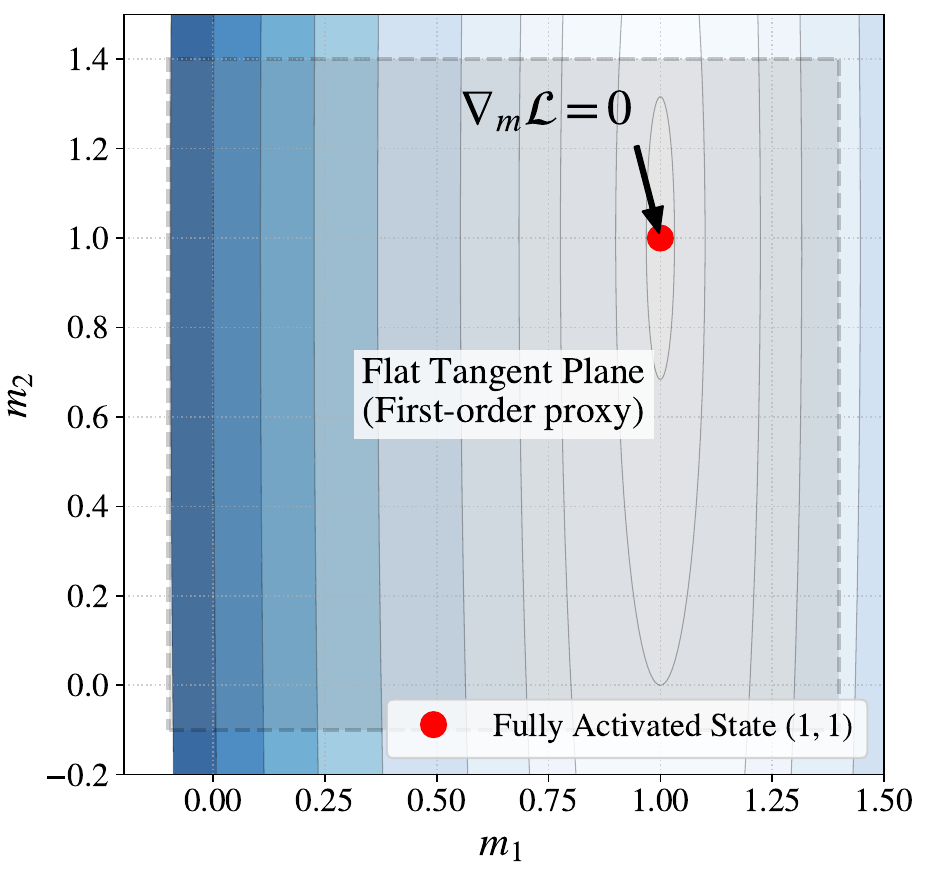}
    \label{fig:moe_2d_contour_a}
  }
  \hspace{2em}
  \subfloat[Hessian (Second-Order)]{
    \includegraphics[width=0.30\textwidth]{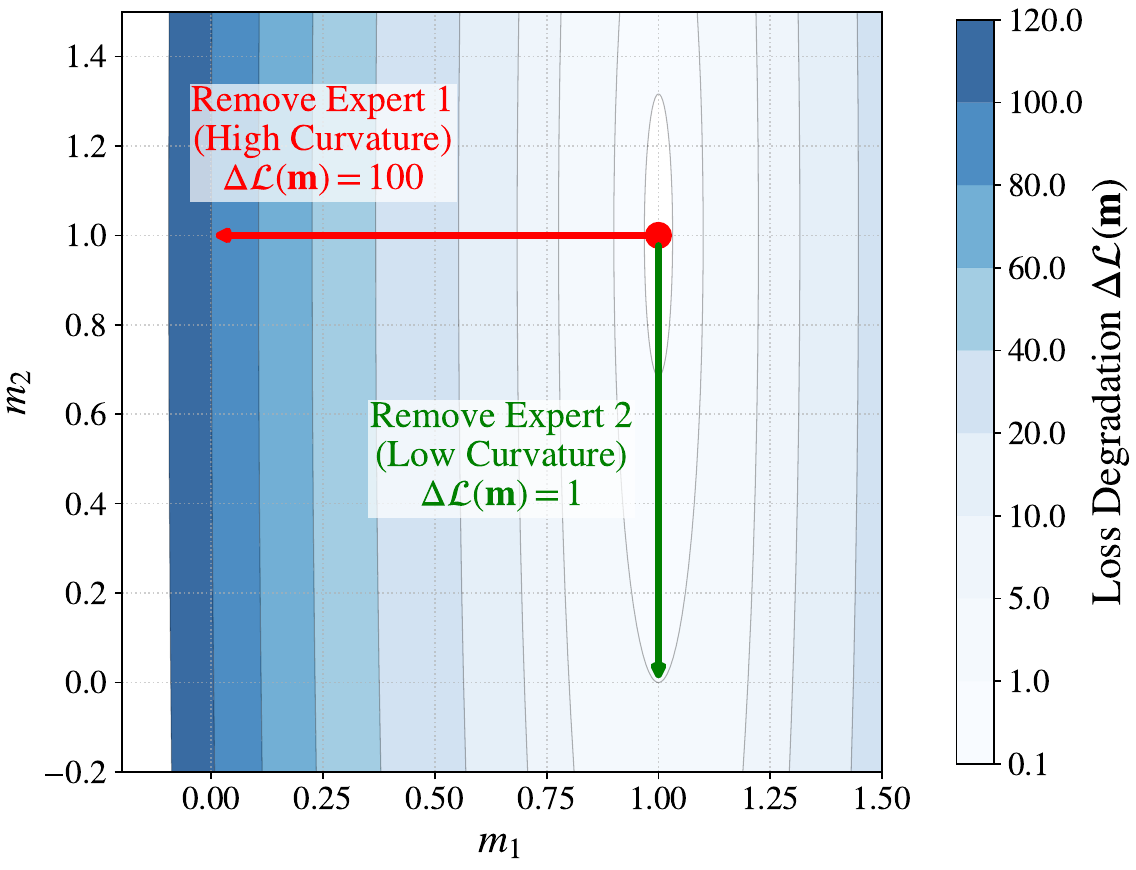}
    \label{fig:moe_2d_contour_b}
  }
  \caption{{ First-order gradients versus second-order Hessian. (a) First-order proxy fits a flat tangent plane at a stationary point, failing to distinguish between experts. (b) Hessian captures steep curvature, revealing the massive penalty of removing expert 1 ($\Delta\mathcal{L}(\mathbf{m}) = 100$) versus expert 2 ($\Delta\mathcal{L}(\mathbf{m}) = 1$).}}
  \label{fig:moe_2d_contour}
  \vspace{-0.7em}
\end{figure}

{\bf{Necessity of Second-Order Curvature.}} While Eqn.~\ref{eqn:loss_expansion} incorporates the first-order gradient term, relying solely on this term is insufficient for expert selection. The first-order term scores each expert separately and thus cannot capture how removing one expert changes the importance of another, which is encoded only by the off-diagonal entries of $\mathbf{H}_{\mathbf{m}}$. Moreover, the first-order term can be uninformative even for individual experts: it fits a flat tangent plane to the loss landscape, failing to capture the steep curvature penalties associated with dropping critical experts. To illustrate this limitation, we consider a simplified MoE landscape $\Delta\mathcal{L}(\mathbf{m}) = 100(m_1 - 1)^2 + (m_2 - 1)^2$, as shown in Figure~\ref{fig:moe_2d_contour}. At the full model $(1, 1)$, the gradient vanishes ($\nabla_{\mathbf{m}} \mathcal{L} = \mathbf{0}$), rendering the first-order term completely uninformative and incorrectly implying that both experts are equally expendable. In contrast, directly evaluating the impact of expert removal reveals a significant disparity: removing expert 1 (high curvature) incurs a substantial loss penalty of $100(0-1)^2 = 100$, whereas removing expert 2 (low curvature) yields a marginal penalty of only $1$. Since the Hessian matrix $\mathbf{H}_{\mathbf{m}}$ explicitly encapsulates both this curvature distinction and the inter-expert coupling, retaining the second-order term in Eqn.~\ref{eqn:loss_expansion} is imperative for establishing a robust selection criterion.

{\bf{Hessian Matrix Derivation.}} Directly computing the exact high-dimensional Hessian matrix $\mathbf{H}_{\mathbf{m}}$ is computationally prohibitive for LLM MoE architectures, as calculating second-order derivatives requires double back-propagation through massive expert networks. To circumvent this bottleneck, we approximate $\mathbf{H}_{\mathbf{m}}$ by the empirical Fisher information matrix~\citep{martens2020new,singh2020woodfisher}, which estimates local curvature using only per-sample first-order gradients:
\begin{equation}\label{eqn:appro_H_m}
  [\mathbf{H}_{\bf m}]_{ij} \approx \mathbb{E}_{x \sim \mathcal{D}}\left[ \frac{\partial\mathcal{L}}{\partial m_i}\cdot\frac{\partial\mathcal{L}}{\partial m_j} \right].
\end{equation}

Recalling the MoE output in Eqn.~\ref{eq:y}, the gradient with respect to the binary state $m_i$ is given by $\frac{\partial \mathcal{L}}{\partial m_i} = \left( \frac{\partial y}{\partial m_i} \right)^\top \nabla_y \mathcal{L} = G_i(x) (E_i(x)^\top \nabla_y \mathcal{L})$. We substitute this into Eqn.~\ref{eqn:appro_H_m}, yielding the following analytical decomposition:
\begin{equation}\label{eqn:appro_H_m_2}
[\mathbf{H}_{\bf m}]_{ij} \approx \mathbb{E}_{x}\left[ \underbrace{G_i(x)G_j(x)}_{\text{Routing Co-activation}} \cdot \underbrace{(E_i(x)^\top\nabla_y\mathcal{L})(E_j(x)^\top\nabla_y\mathcal{L})}_{\text{Gradient Alignment}} \right].
\end{equation}

This formulation decouples the expert interaction into topological and geometric components. The first term on the right-hand side, $G_i(x)G_j(x)$, acts as a topological gate that is non-zero only when experts $i$ and $j$ are co-activated for the same input $x$, thereby capturing sparse routing co-occurrence. The second term captures geometric correlations, where $E_i(x)^\top\nabla_y\mathcal{L}$ represents the functional contribution of expert $i$, quantifying the projection of its output onto the gradient to indicate its effectiveness in reducing the training loss. Consequently, when two co-activated experts exhibit identically signed projections (i.e., driving the loss descent along the consistent optimization trajectory), it implies functional redundancy and results in a positive curvature penalty ($[\mathbf{H}_{\bf m}]_{ij}>0$). Conversely, oppositely signed projections imply complementary error cancellation where one expert provides reverse compensation for the other, yielding a negative curvature reward ($[\mathbf{H}_{\bf m}]_{ij}<0$) that signifies functional synergy.

\subsection{Difference-of-Submodular Decomposition}\label{sec:np_hard_main}

{\bf{Problem Reformulation and Utility Function.}} With the analytical form of $[\mathbf{H}_{\bf m}]_{ij}$ established, we substitute it into the loss degradation approximation to construct a tractable optimization objective. Our goal is to find the binary selection vector ${\bf m}\in \{0,1\}^N$ that minimizes the loss degradation, which is mathematically equivalent to maximizing the retained utility $\mathcal{U}({\bf m})=-\Delta\mathcal{L}(\mathbf{m})$. As derived in Appendix~\ref{sec:utility_derivation}, by expanding the quadratic term and discarding the constant term independent of $\bf{m}$, we reformulate the maximization of the retained utility as follows:
\begin{equation}\label{problem:utiliy}
   \max_{{\bf m} \in \{0,1\}^N} \mathcal{U}({\bf m}) = \sum_{i=1}^N m_i d_i - \frac{1}{2} \sum_{i=1}^N \sum_{j=1}^N m_i m_j [\mathbf{H}_{\bf m}|_{\mathbf{m}=\mathbf{1}}]_{ij}, \quad \text{s.t.} \quad \mathbf{1}^\top\mathbf{m} = k,
\end{equation}
where $d_i = {\sum_{j=1}^N [\mathbf{H}_{\bf m}|_{\mathbf{m}=\mathbf{1}}]_{ij}} - {[\nabla_{\bf m}\mathcal{L}|_{\mathbf{m}=\mathbf{1}}]_i}$ represents the individual utility of expert $E_i$. Here (the evaluation point $|_{\mathbf{m}=\mathbf{1}}$ is omitted in the text for brevity), $- {[\nabla_{\bf m}\mathcal{L}]_i}$ captures the standalone contribution of $E_i$ toward minimizing the loss, while ${\sum_{j=1}^N [\mathbf{H}_{\bf m}]_{ij}}$ aggregates the pairwise coupling effects between $E_i$ and all other experts. The quadratic term $- \frac{1}{2} \sum_{i=1}^N \sum_{j=1}^N m_i m_j [\mathbf{H}_{\bf m}]_{ij}$ explicitly governs the combinatorial interaction among the selected experts. Eqn.~\ref{problem:utiliy} casts the original loss minimization as a cardinality-constrained quadratic binary optimization (QBO) problem, thereby decoupling the individual capability of each expert from their inter-expert dependencies.

{\bf{Difference-of-Submodular Decomposition.}} While the cardinality-constrained QBO problem in~\eqref{problem:utiliy} provides an explicit structural interpretation, directly maximizing the objective $\mathcal{U}(\mathbf{m})$ is computationally intractable, as proved to be an NP-hard problem in Appendix~\ref{sec:np_hard}. This bottleneck stems from the mixed signs of the off-diagonal entries of $\mathbf{H}_{\bf m}$, where positive entries indicate redundancy and negative entries indicate synergy. This mixture of conflicting combinatorial forces renders the objective landscape highly non-convex and non-submodular, preventing the use of standard greedy or convex optimization solvers.

To overcome this bottleneck, we isolate these conflicting coupling effects by decomposing the Hessian $\mathbf{H}_{\bf m}$ into its non-negative components. We define $\mathbf{H}_{\bf m}^+$ and $\mathbf{H}_{\bf m}^-$ such that $[\mathbf{H}_{\bf m}^+]_{ij} = \max(0, [\mathbf{H}_{\bf m}]_{ij})$ and $[\mathbf{H}_{\bf m}^-]_{ij} = \max(0, -[\mathbf{H}_{\bf m}]_{ij})$. By substituting $\mathbf{H}_{\bf m} = \mathbf{H}_{\bf m}^+ - \mathbf{H}_{\bf m}^-$ into the utility function $\mathcal{U}(\mathbf{m})$, we reformulate the objective as the difference between a gain function $f(\mathbf{m})$ and a cost function $g(\mathbf{m})$:
\begin{equation}\label{eq:transformed_problem_eqn}
\small
\max_{\mathbf{m} \in \{0,1\}^N} \!\!\!\!\!\!\!\!\quad \mathcal{U}(\mathbf{m})\! = \!\!\underbrace{\left( \!\sum_{i=1}^N m_i d_i\! +\! \frac{1}{2} \!\sum_{i=1}^N \!\sum_{j=1}^N m_i m_j [\mathbf{H}_{\bf m}^-|_{\mathbf{m}=\mathbf{1}}]_{ij} \!\!\right)}_{f(\mathbf{m})} \!\!-\!\! \underbrace{\left( \!\frac{1}{2}\! \sum_{i=1}^N \sum_{j=1}^N m_i m_j [\mathbf{H}_{\bf m}^+|_{\mathbf{m}=\mathbf{1}}]_{ij} \right)}_{g(\mathbf{m})}\!, \quad \!\!\!\!\text{s.t.} \quad \!\!\!\! \mathbf{1}^\top\mathbf{m} = k,
\end{equation}
where the gain function $f(\mathbf{m})$ represents the cumulative positive utility, incorporating individual expert capability via $d_i$ and functional synergy arising from complementary interactions captured by $\mathbf{H}_{\bf m}^-$; the cost function $g(\mathbf{m})$ quantifies the functional redundancy, acting as a penalty term to discourage the co-selection of overlapping experts via $\mathbf{H}_{\bf m}^+$. The non-negativity of the decomposed components $\mathbf{H}_{\mathbf{m}}^+$ and $\mathbf{H}_{\mathbf{m}}^-$ imparts a mathematical property: both the gain function $f(\mathbf{m})$ and the cost function $g(\mathbf{m})$ are supermodular, and $g(\mathbf{m})$ is additionally monotonically non-decreasing (proved in Appendix~\ref{sec:supermodular}). This guarantee confirms that our objective $\mathcal{U}(\mathbf{m})$ is the difference of two supermodular functions. Since the negation of a supermodular function is submodular, the objective $f(\mathbf{m}) - g(\mathbf{m})$ is algebraically equivalent to $(-g(\mathbf{m})) - (-f(\mathbf{m}))$. This fundamental equivalence converts the original mixed-sign quadratic maximization into a canonical DS optimization problem, which admits efficient MM-type algorithms~\citep{iyer2012algorithms}.

\vspace{-0.4em}
\subsection{Majorization-Minimization Optimization}
\vspace{-0.3em}
{\bf{MM Algorithm.}} To efficiently solve the maximization problem in~\eqref{eq:transformed_problem_eqn}, we employ an MM-style algorithm. The core principle is to iteratively maximize a tractable modular surrogate function $\mathcal{Q}({\bf m}; {\bf m}')$ that coincides with the objective $\mathcal{U}({\bf m})$ at the given selection vector ${\bf m}'$. Following the modular approximation strategy for DS optimization~\citep{iyer2012algorithms}, we construct this surrogate from a modular approximation of the gain function $f(\mathbf{m})$ and a modular upper bound of the cost function $g(\mathbf{m})$.

{\textbf{Linearization of Gain Function.}} We approximate the gain function $f({\bf m})$ by its first-order Taylor expansion at ${\bf m}'$:
\begin{equation}\label{f_m_eq_1}
    f({\bf m}) \approx f({\bf m}') + \sum_{i=1}^N [\nabla_{\bf m} f({\bf m}')]_i ({m}_i - {m}_i') \triangleq L_f({\bf m}; {\bf m}'),
\end{equation}
where the gradient $[\nabla_{\bf m} f({\bf m}')]_i = d_i + \sum_{j=1}^N [{\bf H}_{{\bf m}}^{-}|_{\mathbf{m}=\mathbf{1}}]_{ij} m_j'$ captures the synergy gain of expert $E_i$ relative to the selection vector ${\bf m}'$. Since $f({\bf m})$ is quadratic, the approximation error is exactly $\frac{1}{2}({\bf m}-{\bf m}')^\top {\bf H}_{{\bf m}}^{-}|_{\mathbf{m}=\mathbf{1}} ({\bf m}-{\bf m}')$, which vanishes at ${\bf m}={\bf m}'$ but may be negative elsewhere because ${\bf H}_{{\bf m}}^{-}$ is not necessarily positive semidefinite.

{\textbf{Upper Bound for Cost Function.}} To derive an upper bound $U_g({\bf m}; {\bf m}')$ for the cost function $g({\bf m})$, we leverage the arithmetic equality for binary variables ($m_i m_j = \frac{1}{2}(m_i^2+m_j^2) - \frac{1}{2}(m_i-m_j)^2 = \frac{1}{2}(m_i+m_j) - \frac{1}{2}(m_i-m_j)^2$). This reformulates the cost function $g({\bf m})$ into a linear component and a concave residual. By applying the first-order Taylor expansion to the concave residual around ${\bf m}'$, we obtain a tight upper bound (proved in Appendix~\ref{sec:cost_upper}):
\begin{equation}\label{g_m_eq_1}
    g(\mathbf{m}) \le \sum_{i=1}^N m_i \left[ \sum_{j=1}^N [\mathbf{H}_{\mathbf{m}}^{+}|_{\mathbf{m}=\mathbf{1}}]_{ij} \left( \frac{1}{2} - m_i' + m_j' \right) \right] + C_g' \triangleq U_g(\mathbf{m}; \mathbf{m}'),  
\end{equation}
where $C_g' = \frac{1}{4}\sum_{i=1}^N\sum_{j=1}^N (m_i' - m_j')^2 [\mathbf{H}_{\mathbf{m}}^+|_{\mathbf{m}=\mathbf{1}}]_{ij}$ is a state-dependent constant independent of the selection vector $\mathbf{m}$, and $\sum_{j=1}^N [\mathbf{H}_{\mathbf{m}}^{+}|_{\mathbf{m}=\mathbf{1}}]_{ij} \left( \frac{1}{2} - m_i' + m_j' \right)$ gives the context-aware redundancy penalty representing the potential functional overlap of expert $E_i$ relative to the entire selected expert subset, which heavily penalizes an expert for overlapping with currently active experts, while forgiving overlaps with inactive ones.

By synthesizing the above two functions, we construct the tractable modular surrogate $\mathcal{Q}({\bf m}; {\bf m}')$ for $\mathcal{U}({\bf m})$ as follows:
\begin{equation}\label{eqn:surro_func}
\mathcal{Q}({\bf m}; {\bf m}') \triangleq L_f({\bf m}; {\bf m}') - U_g({\bf m}; \mathbf{m}'), \quad \text{with} \quad \mathcal{Q}({\bf m}'; {\bf m}') = \mathcal{U}({\bf m}').
\end{equation}
Based on the surrogate function $\mathcal{Q}({\bf m}; {\bf m}')$ in Eqn.~\ref{eqn:surro_func}, we establish the iterative optimization procedure. Let ${\bf m}^{(t)}$ denote the expert selection vector at the $t$-th iteration. By setting the reference point ${\bf m}' = {\bf m}^{(t)}$ and discarding constant terms independent of $\bf m$, the optimization problem at the $(t+1)$-th iteration reduces to a cardinality-constrained linear maximization (proved in Appendix~\ref{sec:card_constrain}):
\begin{equation}\label{eq:mm_update}
    {\bf m}^{(t+1)} = \mathop{\arg\max}_{{\bf m}\in\{0,1\}^N:\,\mathbf{1}^\top{\bf m} = k} \sum_{i=1}^{N} m_i S_i^{(t)},
\end{equation}
where $S_i^{(t)} = [\nabla_{\mathbf{m}} f(\mathbf{m}^{(t)})]_i - \sum_{j=1}^{N} [\mathbf{H}_{\mathbf{m}}^{+}|_{\mathbf{m}=\mathbf{1}}]_{ij} \left( \frac{1}{2} - m_i^{(t)} + m_j^{(t)} \right)$ represents the importance score of $E_i$, which quantifies the trade-off between its synergy gain and redundancy penalty.

Each iteration therefore reduces to a simple Top-$k$ sorting: the maximizer ${\bf m}^{(t+1)}$ is determined by selecting the $k$ experts with the highest scores. Since $\mathcal{Q}$ does not necessarily lower-bound $\mathcal{U}$, we accept ${\bf m}^{(t+1)}$ only if $\mathcal{U}({\bf m}^{(t+1)}) > \mathcal{U}({\bf m}^{(t)})$ and otherwise terminate with ${\bf m}^{(t)}$, which costs only $\mathcal{O}(N^2)$ because $\mathcal{U}$ is an explicit quadratic function; the first iterate ${\bf m}^{(1)}$ is always accepted to initialize a feasible selection. The complete procedure is outlined in Appendix~\ref{sec:mm_algo}, and its convergence is guaranteed by the following theorem:
\begin{theorem}[Convergence Properties]\label{thm:convergence}
The sequence of expert selection vectors $\{{\bf m}^{(t)}\}_{t \ge 1}$ generated by the MM algorithm yields a monotonically non-decreasing objective value. Specifically, the following inequality holds:
\begin{equation}\label{eq:U_m_inequlity}
 \mathcal{U}({\bf m}^{(t+1)}) \ge \mathcal{U}({\bf m}^{(t)}), \quad \forall t \ge 1,
\end{equation}
where the inequality is strict whenever the selection changes. Given that the feasible region of ${\bf m}$ is finite, no selection is visited twice, and the MM algorithm terminates within a finite number of iterations.
\end{theorem}
{\textit{Proof. See Appendix~\ref{sec:mm_convergence}.}}

\vspace{-0.2em}
\section{Experiments}
\label{sec:experiments}
\vspace{-0.6em}

\begin{table*}[t]
\centering
\caption{Experimental results on test accuracy (\%) ($\uparrow$) and convergence rounds ($\downarrow$). We evaluate DS-MoE and five baselines on Switch-Base-32, Switch-Base-64, and DeepSeek-MoE-16B.}
\label{tab:main_results}
\resizebox{0.98\linewidth}{!}{
\begin{tabular}{l|l|ccccc|ccccc}
\toprule
\multirow{2}{*}{Model}& \multirow{2}{*}{Method}& \multicolumn{5}{c|}{Test Accuracy (\%) $\uparrow$} & \multicolumn{5}{c}{Convergence Rounds $\downarrow$} \\
\cmidrule(lr){3-7} \cmidrule(lr){8-12}
 & & AGNews & PIQA & HellaSwag & MMLU & Avg. & AGNews & PIQA & HellaSwag & MMLU & Avg. \\
\midrule
\multirow{11}{*}[2ex]{Switch-Base-32} & Rand-MoE  
  & 87.3 $\pm$ 0.3 & 68.1 $\pm$ 0.8 & 65.9 $\pm$ 0.7 & 39.8 $\pm$ 1.1 & 65.3 $\pm$ 0.7 & 84 $\pm$ 6 & 26 $\pm$ 3 & 28 $\pm$ 4 & 34 $\pm$ 5 & 43.0 $\pm$ 4.5 \\
  \cmidrule(lr){2-12}
  & S-MoE      
  & 92.6 $\pm$ 0.2 & 74.3 $\pm$ 0.5 & 71.8 $\pm$ 0.4 & 46.4 $\pm$ 0.6 & 71.3 $\pm$ 0.4 & 59 $\pm$ 3 & 15 $\pm$ 2 & 17 $\pm$ 2 & 19 $\pm$ 3 & 27.5 $\pm$ 2.5 \\
  \cmidrule(lr){2-12}
  & SEER-MoE  
  & 93.1 $\pm$ 0.3 & 75.2 $\pm$ 0.4 & 72.7 $\pm$ 0.5 & 46.9 $\pm$ 0.5 & 72.0 $\pm$ 0.4 & 49 $\pm$ 3 & 12 $\pm$ 1 & 13 $\pm$ 2 & 17 $\pm$ 2 & 22.8 $\pm$ 2.0 \\
  \cmidrule(lr){2-12}
    & DiEP      
  & 93.3 $\pm$ 0.2 & 75.8 $\pm$ 0.3 & 73.2 $\pm$ 0.4 & 47.5 $\pm$ 0.3 & 72.5 $\pm$ 0.3 & 48 $\pm$ 2 & 12 $\pm$ 2 & 12 $\pm$ 1 & 15 $\pm$ 2 & 21.8 $\pm$ 1.8 \\
  \cmidrule(lr){2-12}
    & REAP      
  & 93.5 $\pm$ 0.2 & 76.4 $\pm$ 0.4 & 73.1 $\pm$ 0.3 & 48.1 $\pm$ 0.2 & 72.8 $\pm$ 0.3 & 45 $\pm$ 3 & 11 $\pm$ 2 & 13 $\pm$ 2 & 14 $\pm$ 1 & 20.8 $\pm$ 2.0 \\
  \cmidrule(lr){2-12}
  & \cellcolor{gray!30}\textbf{DS-MoE} 
  & \cellcolor{gray!30}\textbf{94.7} $\pm$ \textbf{0.1} & \cellcolor{gray!30}\textbf{78.6} $\pm$ \textbf{0.2} & \cellcolor{gray!30}\textbf{75.6} $\pm$ \textbf{0.2} & \cellcolor{gray!30}\textbf{51.8} $\pm$ \textbf{0.3} & \cellcolor{gray!30}\textbf{75.2} $\pm$ \textbf{0.2} & \cellcolor{gray!30}\textbf{32} $\pm$ \textbf{1} & \cellcolor{gray!30}\textbf{8} $\pm$ \textbf{1} & \cellcolor{gray!30}\textbf{10} $\pm$ \textbf{1} & \cellcolor{gray!30}\textbf{11} $\pm$ \textbf{2} & \cellcolor{gray!30}\textbf{15.3} $\pm$ \textbf{1.3} \\
\midrule
\multirow{11}{*}[2ex]{Switch-Base-64} & Rand-MoE  
  & 87.6 $\pm$ 0.4 & 68.7 $\pm$ 0.6 & 66.1 $\pm$ 0.9 & 40.3 $\pm$ 1.2 & 65.7 $\pm$ 0.8 & 82 $\pm$ 7 & 23 $\pm$ 4 & 25 $\pm$ 3 & 26 $\pm$ 5 & 39.0 $\pm$ 4.8 \\
  \cmidrule(lr){2-12}
  & S-MoE      
  & 92.8 $\pm$ 0.3 & 74.8 $\pm$ 0.6 & 72.1 $\pm$ 0.5 & 46.8 $\pm$ 0.8 & 71.6 $\pm$ 0.6 & 55 $\pm$ 4 & 13 $\pm$ 2 & 14 $\pm$ 3 & 16 $\pm$ 2 & 24.5 $\pm$ 2.8 \\
  \cmidrule(lr){2-12}
  & SEER-MoE  
  & 93.2 $\pm$ 0.2 & 76.3 $\pm$ 0.5 & 73.3 $\pm$ 0.4 & 47.9 $\pm$ 0.6 & 72.7 $\pm$ 0.4 & 47 $\pm$ 2 & 11 $\pm$ 2 & 13 $\pm$ 1 & 14 $\pm$ 3 & 21.3 $\pm$ 2.0 \\
  \cmidrule(lr){2-12}
    & DiEP      
  & 93.5 $\pm$ 0.2 & 76.5 $\pm$ 0.3 & 73.9 $\pm$ 0.4 & 48.1 $\pm$ 0.5 & 73.0 $\pm$ 0.4 & 45 $\pm$ 3 & 10 $\pm$ 1 & 11 $\pm$ 2 & 13 $\pm$ 1 & 19.8 $\pm$ 1.8 \\
  \cmidrule(lr){2-12}
    & REAP      
  & 93.9 $\pm$ 0.2 & 76.7 $\pm$ 0.5 & 74.5 $\pm$ 0.3 & 48.2 $\pm$ 0.4 & 73.3 $\pm$ 0.4 & 41 $\pm$ 2 & 9 $\pm$ 2 & 9 $\pm$ 1 & 11 $\pm$ 2 & 17.5 $\pm$ 1.8 \\
  \cmidrule(lr){2-12}
  & \cellcolor{gray!30}\textbf{DS-MoE} 
  & \cellcolor{gray!30}\textbf{94.9} $\pm$ \textbf{0.1} & \cellcolor{gray!30}\textbf{78.9} $\pm$ \textbf{0.2} & \cellcolor{gray!30}\textbf{75.8} $\pm$ \textbf{0.3} & \cellcolor{gray!30}\textbf{52.4} $\pm$ \textbf{0.4} & \cellcolor{gray!30}\textbf{75.5} $\pm$ \textbf{0.3} & \cellcolor{gray!30}\textbf{30} $\pm$ \textbf{2} & \cellcolor{gray!30}\textbf{7} $\pm$ \textbf{1} & \cellcolor{gray!30}\textbf{8} $\pm$ \textbf{1} & \cellcolor{gray!30}\textbf{10} $\pm$ \textbf{1} & \cellcolor{gray!30}\textbf{13.8} $\pm$ \textbf{1.3} \\
\midrule
\multirow{11}{*}[2ex]{DeepSeek-MoE-16B} & Rand-MoE  
  & 86.1 $\pm$ 0.5 & 69.8 $\pm$ 0.9 & 68.2 $\pm$ 0.8 & 32.8 $\pm$ 1.3 & 64.2 $\pm$ 0.9 & 138 $\pm$ 9 & 33 $\pm$ 6 & 35 $\pm$ 5 & 39 $\pm$ 7 & 61.3 $\pm$ 6.8 \\
  \cmidrule(lr){2-12}
  & S-MoE 
  & 92.1 $\pm$ 0.4 & 75.2 $\pm$ 0.7 & 73.4 $\pm$ 0.6 & 39.1 $\pm$ 0.9 & 70.0 $\pm$ 0.7 & 85 $\pm$ 6 & 20 $\pm$ 3 & 18 $\pm$ 4 & 22 $\pm$ 3 & 36.3 $\pm$ 4.0 \\
  \cmidrule(lr){2-12}
  & SEER-MoE  
  & 92.3 $\pm$ 0.3 & 77.1 $\pm$ 0.5 & 74.6 $\pm$ 0.6 & 40.3 $\pm$ 0.7 & 71.1 $\pm$ 0.5 & 77 $\pm$ 4 & 17 $\pm$ 3 & 17 $\pm$ 2 & 18 $\pm$ 3 & 32.3 $\pm$ 3.0 \\
  \cmidrule(lr){2-12}
    & DiEP      
  & 92.9 $\pm$ 0.2 & 77.6 $\pm$ 0.4 & 75.1 $\pm$ 0.5 & 42.1 $\pm$ 0.6 & 71.9 $\pm$ 0.4 & 73 $\pm$ 3 & 15 $\pm$ 2 & 16 $\pm$ 3 & 17 $\pm$ 2 & 30.3 $\pm$ 2.5 \\
  \cmidrule(lr){2-12}
    & REAP      
  & 93.1 $\pm$ 0.3 & 78.3 $\pm$ 0.5 & 75.4 $\pm$ 0.6 & 42.6 $\pm$ 0.5 & 72.3 $\pm$ 0.5 & 70 $\pm$ 4 & 14 $\pm$ 2 & 15 $\pm$ 2 & 15 $\pm$ 3 & 28.5 $\pm$ 2.8 \\
  \cmidrule(lr){2-12}
  &\cellcolor{gray!30} \textbf{DS-MoE} 
  & \cellcolor{gray!30}\textbf{94.1} $\pm$ \textbf{0.2} & \cellcolor{gray!30}\textbf{81.2} $\pm$ \textbf{0.3} & \cellcolor{gray!30}\textbf{77.9} $\pm$ \textbf{0.2} & \cellcolor{gray!30}\textbf{45.1} $\pm$ \textbf{0.4} & \cellcolor{gray!30}\textbf{74.6} $\pm$ \textbf{0.3} & \cellcolor{gray!30}\textbf{45} $\pm$ \textbf{2} & \cellcolor{gray!30}\textbf{12} $\pm$ \textbf{1} & \cellcolor{gray!30}\textbf{13} $\pm$ \textbf{1} & \cellcolor{gray!30}\textbf{14} $\pm$ \textbf{2} & \cellcolor{gray!30}\textbf{21.0} $\pm$ \textbf{1.5} \\
\bottomrule

\end{tabular}
}
\vspace{-0.2em}
\end{table*}

\subsection{Experimental Setup}
\label{subsec:exp_setup}
\vspace{-0.5em}
{\bf{Tasks and Datasets.}} We conduct extensive experiments across three tasks: i) {{\textit{Natural Language Understanding}}} on AGNews~\cite{zhang2015character} dataset; ii) {{\textit{Commonsense Reasoning}}} on PIQA~\cite{bisk2020piqa} and HellaSwag~\cite{zellers2019hellaswag} datasets; and iii) {{\textit{Advanced Multi-domain Reasoning}}} on  MMLU~\cite{hendrycksmeasuring} dataset. Comprehensive dataset descriptions and experimental setups are detailed in Appendix~\ref{sec:datasets}.

\textbf{Models.} We evaluate the proposed DS-MoE framework on three representative large-scale MoE architectures: \textit{Switch-Base-32}, \textit{Switch-Base-64}~\cite{fedus2022switch}, and \textit{DeepSeek-MoE-16B}~\cite{dai2024deepseekmoe}. The implementation details are shown in Appendix~\ref{sec:hyperparams}.

% \textbf{Models.} We evaluate the proposed DS-MoE framework on two representative large-scale MoE architectures: \textit{Switch-Base-32} and \textit{Switch-Base-64}, which distributes computation across 64 and 128 experts per layer to enable efficient scaling to massive model parameters while maintaining a sparse activation footprint, and \textit{DeepSeek-MoE-16B} with 8-rank QLoRA, which leverages a highly expressive sparse routing backbone to achieve performance comparable to the dense LLaMA2-7B model. The implementation details are shown in Appendix~\ref{appendix:model}.

\textbf{Baselines.} We compare CLR-MoE against five representative baselines: i) {\textit{REAP}}~\cite{lasby2026reap}: A one-shot selection method that scores each expert by its layer-wise router-weighted activation norm; ii) {\textit{DiEP}}~\cite{baidiep}: An expert selection strategy learn intra-layer expert importance to drop redundant experts via gradient-based evaluation; iii) {\textit{SEER-MoE}}~\cite{muzio2024seer}: A MoE framework that selects experts by heavy-hitter routing counts and then fine-tunes with regularization; iv) {\textit{Standard MoE (S-MoE)}}~\cite{shazeer2017outrageously}: a routing mechanism that activates a fixed number of experts per token based on top gating probabilities; and v) {\textit{Random MoE (Rand-MoE):}} A data-independent baseline that randomly samples active experts without relying on routing or gradient signals.

\vspace{-1em}

\subsection{Main Results}

\vspace{-1em}

\textbf{Natural Language Understanding.} Table~\ref{tab:main_results} shows DS-MoE consistently outperforms all baselines on the AGNews dataset, achieving over 94\% test accuracies across Switch-Base-32, Switch-Base-64, and DeepSeek-MoE-16B. This performance consistently surpasses both standard Top-$k$ routing (S-MoE) and advanced baselines (SEER-MoE, DiEP). Furthermore, by penalizing functional redundancy via positive Hessian entries, DS-MoE drastically expedites model training. On DeepSeek-MoE-16B, DS-MoE converges in just 45 rounds—reducing S-MoE's required 100 rounds by more than half, while maintaining a substantial speed advantage over both DiEP and SEER-MoE.

\textbf{Commonsense Reasoning.} Table~\ref{tab:main_results} shows that DS-MoE exhibits significant superiority on complex reasoning tasks PIQA and HellaSwag. The first-order methods like DiEP struggle in this regime due to the lack of capturing inter-expert dependencies. By contrast, DS-MoE exploits second-order interactions to identify synergistic expert combinations, reaches 78.6\% and 75.6\% test accuracy on Switch-Base-32, and outperforms the standard S-MoE as well as SEER-MoE and DiEP. Beyond accuracy improvements, DS-MoE accelerates model convergence for these tasks by 27\% to 46\% compared to all baselines.

\begin{table*}[t]
\centering
\caption{Robustness and Scalability Analysis on DeepSeek-MoE-16B. The model performance versus (a) expert sparsity $k$ with $N=64$ and (b) expert pool size $N$ with $k=2$.
}
\label{tab:micro_benchmarking}

% ==================== (a) ====================
\vspace{0.1cm}
{\centering \small {(a) Performance versus expert sparsity $k$} \par}
\vspace{0.1cm}

\resizebox{0.98\linewidth}{!}{
\begin{tabular}{l|c|l|ccccc|ccccc}
\toprule
\multirow{2}{*}{Model} & \multirow{2}{*}{$k$} & \multirow{2}{*}{Method} & \multicolumn{5}{c|}{Test Accuracy (\%) $\uparrow$} & \multicolumn{5}{c}{Convergence Rounds $\downarrow$} \\
\cmidrule(lr){4-8} \cmidrule(lr){9-13}
& & & AGNews & PIQA & HellaSwag & MMLU & Avg. & AGNews & PIQA & HellaSwag & MMLU & Avg. \\
\midrule
\multirow{18}{*}{DeepSeek-16B} 
& \multirow{6}{*}{1} 
& Rand-MoE & 83.2 $\pm$ 0.6 & 66.8 $\pm$ 1.2 & 65.9 $\pm$ 0.9 & 28.9 $\pm$ 1.5 & 61.2 $\pm$ 1.1 & 162 $\pm$ 14 & 59 $\pm$ 7 & 56 $\pm$ 9 & 57 $\pm$ 11 & 83.5 $\pm$ 10.3 \\
& & S-MoE & 90.1 $\pm$ 0.5 & 73.4 $\pm$ 0.9 & 72.1 $\pm$ 0.7 & 37.7 $\pm$ 1.2 & 68.3 $\pm$ 0.8 & 118 $\pm$ 9 & 26 $\pm$ 4 & 23 $\pm$ 5 & 29 $\pm$ 6 & 49.0 $\pm$ 6.0 \\
& & SEER-MoE & 90.9 $\pm$ 0.4 & 75.7 $\pm$ 0.7 & 73.2 $\pm$ 0.8 & 39.2 $\pm$ 0.9 & 69.8 $\pm$ 0.7 & 112 $\pm$ 6 & 19 $\pm$ 4 & 19 $\pm$ 3 & 21 $\pm$ 4 & 42.8 $\pm$ 4.3 \\
& & DiEP & 91.6 $\pm$ 0.3 & 75.9 $\pm$ 0.6 & 73.8 $\pm$ 0.7 & 40.8 $\pm$ 0.8 & 70.5 $\pm$ 0.6 & 104 $\pm$ 5 & 18 $\pm$ 3 & 17 $\pm$ 4 & 18 $\pm$ 3 & 39.3 $\pm$ 3.8 \\
& & REAP & 91.8 $\pm$ 0.3 & 76.8 $\pm$ 0.6 & 73.6 $\pm$ 0.6 & 41.2 $\pm$ 0.7 & 70.9 $\pm$ 0.6 & 99 $\pm$ 6 & 16 $\pm$ 3 & 16 $\pm$ 4 & 17 $\pm$ 3 & 37.0 $\pm$ 4.0 \\
& & \cellcolor{gray!30}\textbf{DS-MoE} & \cellcolor{gray!30}\textbf{93.2} $\pm$ \textbf{0.3} & \cellcolor{gray!30}\textbf{79.8} $\pm$ \textbf{0.4} & \cellcolor{gray!30}\textbf{76.3} $\pm$ \textbf{0.3} & \cellcolor{gray!30}\textbf{43.9} $\pm$ \textbf{0.5} & \cellcolor{gray!30}\textbf{73.3} $\pm$ \textbf{0.4} & \cellcolor{gray!30}\textbf{66} $\pm$ \textbf{3} & \cellcolor{gray!30}\textbf{12} $\pm$ \textbf{2} & \cellcolor{gray!30}\textbf{13} $\pm$ \textbf{2} & \cellcolor{gray!30}\textbf{15} $\pm$ \textbf{3} & \cellcolor{gray!30}\textbf{26.5} $\pm$ \textbf{2.5} \\
\cmidrule(lr){2-13}
& \multirow{6}{*}{2} 
& Rand-MoE & 86.1 $\pm$ 0.5 & 69.8 $\pm$ 0.9 & 68.2 $\pm$ 0.8 & 32.8 $\pm$ 1.3 & 64.2 $\pm$ 0.9 & 138 $\pm$ 9 & 33 $\pm$ 6 & 35 $\pm$ 5 & 39 $\pm$ 7 & 61.3 $\pm$ 6.8 \\
& & S-MoE & 92.1 $\pm$ 0.4 & 75.2 $\pm$ 0.7 & 73.4 $\pm$ 0.6 & 39.1 $\pm$ 0.9 & 70.0 $\pm$ 0.7 & 85 $\pm$ 6 & 20 $\pm$ 3 & 18 $\pm$ 4 & 22 $\pm$ 3 & 36.3 $\pm$ 4.0 \\
& & SEER-MoE & 92.3 $\pm$ 0.3 & 77.1 $\pm$ 0.5 & 74.6 $\pm$ 0.6 & 40.3 $\pm$ 0.7 & 71.1 $\pm$ 0.5 & 77 $\pm$ 4 & 17 $\pm$ 3 & 17 $\pm$ 2 & 18 $\pm$ 3 & 32.3 $\pm$ 3.0 \\
& & DiEP & 92.9 $\pm$ 0.2 & 77.6 $\pm$ 0.4 & 75.1 $\pm$ 0.5 & 42.1 $\pm$ 0.6 & 71.9 $\pm$ 0.4 & 73 $\pm$ 3 & 15 $\pm$ 2 & 16 $\pm$ 3 & 17 $\pm$ 2 & 30.3 $\pm$ 2.5 \\
& & REAP & 93.1 $\pm$ 0.3 & 78.3 $\pm$ 0.5 & 75.4 $\pm$ 0.6 & 42.6 $\pm$ 0.5 & 72.3 $\pm$ 0.5 & 70 $\pm$ 4 & 14 $\pm$ 2 & 15 $\pm$ 2 & 15 $\pm$ 3 & 28.5 $\pm$ 2.8 \\
& & \cellcolor{gray!30}\textbf{DS-MoE} & \cellcolor{gray!30}\textbf{94.1} $\pm$ \textbf{0.2} & \cellcolor{gray!30}\textbf{81.2} $\pm$ \textbf{0.3} & \cellcolor{gray!30}\textbf{77.9} $\pm$ \textbf{0.2} & \cellcolor{gray!30}\textbf{45.1} $\pm$ \textbf{0.4} & \cellcolor{gray!30}\textbf{74.6} $\pm$ \textbf{0.3} & \cellcolor{gray!30}\textbf{45} $\pm$ \textbf{2} & \cellcolor{gray!30}\textbf{12} $\pm$ \textbf{1} & \cellcolor{gray!30}\textbf{13} $\pm$ \textbf{1} & \cellcolor{gray!30}\textbf{14} $\pm$ \textbf{2} & \cellcolor{gray!30}\textbf{21.0} $\pm$ \textbf{1.5} \\
\cmidrule(lr){2-13}
& \multirow{6}{*}{3} 
& Rand-MoE & 86.7 $\pm$ 0.4 & 70.6 $\pm$ 0.7 & 69.4 $\pm$ 0.6 & 34.1 $\pm$ 1.1 & 65.2 $\pm$ 0.7 & 125 $\pm$ 7 & 30 $\pm$ 4 & 31 $\pm$ 4 & 33 $\pm$ 5 & 54.8 $\pm$ 5.0 \\
& & S-MoE & 92.4 $\pm$ 0.3 & 75.9 $\pm$ 0.5 & 73.9 $\pm$ 0.5 & 40.0 $\pm$ 0.7 & 70.6 $\pm$ 0.5 & 76 $\pm$ 5 & 18 $\pm$ 2 & 17 $\pm$ 3 & 19 $\pm$ 2 & 32.5 $\pm$ 3.0 \\
& & SEER-MoE & 92.7 $\pm$ 0.2 & 77.3 $\pm$ 0.4 & 75.1 $\pm$ 0.4 & 41.3 $\pm$ 0.6 & 71.6 $\pm$ 0.4 & 75 $\pm$ 3 & 16 $\pm$ 2 & 15 $\pm$ 1 & 17 $\pm$ 2 & 30.8 $\pm$ 2.0 \\
& & DiEP & 93.1 $\pm$ 0.1 & 78.0 $\pm$ 0.3 & 75.7 $\pm$ 0.3 & 42.6 $\pm$ 0.4 & 72.4 $\pm$ 0.3 & 68 $\pm$ 2 & 13 $\pm$ 1 & 14 $\pm$ 2 & 16 $\pm$ 1 & 27.8 $\pm$ 1.5 \\
& & REAP & 93.4 $\pm$ 0.2 & 78.9 $\pm$ 0.2 & 76.3 $\pm$ 0.3 & 43.4 $\pm$ 0.5 & 73.0 $\pm$ 0.3 & 64 $\pm$ 3 & 13 $\pm$ 2 & 13 $\pm$ 1 & 15 $\pm$ 2 & 26.3 $\pm$ 2.0 \\
& & \cellcolor{gray!30}\textbf{DS-MoE} & \cellcolor{gray!30}\textbf{94.5} $\pm$ \textbf{0.1} & \cellcolor{gray!30}\textbf{81.4} $\pm$ \textbf{0.2} & \cellcolor{gray!30}\textbf{78.2} $\pm$ \textbf{0.1} & \cellcolor{gray!30}\textbf{45.4} $\pm$ \textbf{0.2} & \cellcolor{gray!30}\textbf{74.9} $\pm$ \textbf{0.2} & \cellcolor{gray!30}\textbf{38} $\pm$ \textbf{1} & \cellcolor{gray!30}\textbf{11} $\pm$ \textbf{1} & \cellcolor{gray!30}\textbf{12} $\pm$ \textbf{1} & \cellcolor{gray!30}\textbf{12} $\pm$ \textbf{1} & \cellcolor{gray!30}\textbf{18.3} $\pm$ \textbf{1.0} \\
\bottomrule
\end{tabular}
}

% ==================== (b) ====================
\vspace{0.3cm}
{\centering \small {(b) Performance versus expert pool size $N$} \par}
\vspace{0.1cm}

\resizebox{0.98\linewidth}{!}{
\begin{tabular}{l|c|l|ccccc|ccccc}
\toprule
\multirow{2}{*}{Model} & \multirow{2}{*}{$N$} & \multirow{2}{*}{Method} & \multicolumn{5}{c|}{Test Accuracy (\%) $\uparrow$} & \multicolumn{5}{c}{Convergence Rounds $\downarrow$} \\
\cmidrule(lr){4-8} \cmidrule(lr){9-13}
& & & AGNews & PIQA & HellaSwag & MMLU & Avg. & AGNews & PIQA & HellaSwag & MMLU & Avg. \\
\midrule
\multirow{18}{*}{DeepSeek-16B} 
& \multirow{6}{*}{32} 
& Rand-MoE & 84.3 $\pm$ 0.8 & 67.9 $\pm$ 1.2 & 66.8 $\pm$ 1.0 & 31.2 $\pm$ 1.7 & 62.6 $\pm$ 1.1 & 153 $\pm$ 12 & 41 $\pm$ 8 & 51 $\pm$ 8 & 45 $\pm$ 9 & 72.5 $\pm$ 8.7 \\
& & S-MoE & 90.8 $\pm$ 0.5 & 72.5 $\pm$ 0.8 & 71.7 $\pm$ 0.9 & 38.0 $\pm$ 1.3 & 68.3 $\pm$ 0.8 & 115 $\pm$ 9 & 22 $\pm$ 4 & 24 $\pm$ 5 & 28 $\pm$ 4 & 47.3 $\pm$ 5.0 \\
& & SEER-MoE & 91.4 $\pm$ 0.4 & 75.6 $\pm$ 0.7 & 73.7 $\pm$ 0.8 & 39.3 $\pm$ 0.9 & 70.0 $\pm$ 0.7 & 108 $\pm$ 7 & 19 $\pm$ 4 & 19 $\pm$ 4 & 21 $\pm$ 5 & 41.8 $\pm$ 4.2 \\
& & DiEP & 91.8 $\pm$ 0.4 & 76.2 $\pm$ 0.7 & 74.3 $\pm$ 0.6 & 41.2 $\pm$ 1.0 & 70.9 $\pm$ 0.6 & 103 $\pm$ 5 & 18 $\pm$ 3 & 17 $\pm$ 4 & 17 $\pm$ 4 & 38.8 $\pm$ 3.5 \\
& & REAP & 92.0 $\pm$ 0.3 & 76.5 $\pm$ 0.6 & 74.6 $\pm$ 0.6 & 41.7 $\pm$ 0.8 & 71.2 $\pm$ 0.6 & 98 $\pm$ 4 & 16 $\pm$ 3 & 15 $\pm$ 4 & 16 $\pm$ 3 & 36.3 $\pm$ 3.5 \\
& & \cellcolor{gray!30}\textbf{DS-MoE} & \cellcolor{gray!30}\textbf{93.1} $\pm$ \textbf{0.3} & \cellcolor{gray!30}\textbf{80.0} $\pm$ \textbf{0.5} & \cellcolor{gray!30}\textbf{77.2} $\pm$ \textbf{0.4} & \cellcolor{gray!30}\textbf{44.0} $\pm$ \textbf{0.7} & \cellcolor{gray!30}\textbf{73.6} $\pm$ \textbf{0.4} & \cellcolor{gray!30}\textbf{69} $\pm$ \textbf{4} & \cellcolor{gray!30}\textbf{13} $\pm$ \textbf{2} & \cellcolor{gray!30}\textbf{14} $\pm$ \textbf{2} & \cellcolor{gray!30}\textbf{15} $\pm$ \textbf{3} & \cellcolor{gray!30}\textbf{27.8} $\pm$ \textbf{2.7} \\
\cmidrule(lr){2-13}
& \multirow{6}{*}{48} 
& Rand-MoE & 85.3 $\pm$ 0.6 & 69.0 $\pm$ 1.2 & 67.6 $\pm$ 0.8 & 32.1 $\pm$ 1.4 & 63.5 $\pm$ 1.0 & 144 $\pm$ 10 & 37 $\pm$ 6 & 41 $\pm$ 7 & 43 $\pm$ 8 & 66.3 $\pm$ 7.8 \\
& & S-MoE & 91.5 $\pm$ 0.5 & 74.5 $\pm$ 0.7 & 73.0 $\pm$ 0.8 & 38.6 $\pm$ 1.0 & 69.4 $\pm$ 0.7 & 106 $\pm$ 6 & 20 $\pm$ 4 & 23 $\pm$ 4 & 23 $\pm$ 4 & 43.0 $\pm$ 4.5 \\
& & SEER-MoE & 91.9 $\pm$ 0.4 & 76.5 $\pm$ 0.5 & 74.2 $\pm$ 0.7 & 39.9 $\pm$ 0.9 & 70.6 $\pm$ 0.5 & 102 $\pm$ 5 & 18 $\pm$ 3 & 17 $\pm$ 3 & 19 $\pm$ 3 & 39.0 $\pm$ 3.5 \\
& & DiEP & 92.4 $\pm$ 0.2 & 77.0 $\pm$ 0.6 & 74.8 $\pm$ 0.6 & 41.7 $\pm$ 0.8 & 71.5 $\pm$ 0.5 & 76 $\pm$ 4 & 17 $\pm$ 3 & 17 $\pm$ 3 & 17 $\pm$ 2 & 31.8 $\pm$ 3.0 \\
& & REAP & 92.7 $\pm$ 0.3 & 77.3 $\pm$ 0.6 & 75.0 $\pm$ 0.5 & 42.3 $\pm$ 0.6 & 71.8 $\pm$ 0.5 & 73 $\pm$ 3 & 15 $\pm$ 3 & 16 $\pm$ 4 & 15 $\pm$ 2 & 29.8 $\pm$ 3.0 \\
& & \cellcolor{gray!30}\textbf{DS-MoE} & \cellcolor{gray!30}\textbf{93.7} $\pm$ \textbf{0.3} & \cellcolor{gray!30}\textbf{80.7} $\pm$ \textbf{0.4} & \cellcolor{gray!30}\textbf{77.6} $\pm$ \textbf{0.4} & \cellcolor{gray!30}\textbf{44.6} $\pm$ \textbf{0.5} & \cellcolor{gray!30}\textbf{74.2} $\pm$ \textbf{0.3} & \cellcolor{gray!30}\textbf{52} $\pm$ \textbf{3} & \cellcolor{gray!30}\textbf{13} $\pm$ \textbf{2} & \cellcolor{gray!30}\textbf{12} $\pm$ \textbf{1} & \cellcolor{gray!30}\textbf{14} $\pm$ \textbf{2} & \cellcolor{gray!30}\textbf{22.8} $\pm$ \textbf{2.3} \\
\cmidrule(lr){2-13}
& \multirow{6}{*}{64} 
& Rand-MoE & 86.1 $\pm$ 0.5 & 69.8 $\pm$ 0.9 & 68.2 $\pm$ 0.8 & 32.8 $\pm$ 1.3 & 64.2 $\pm$ 0.9 & 138 $\pm$ 9 & 33 $\pm$ 6 & 35 $\pm$ 5 & 39 $\pm$ 7 & 61.3 $\pm$ 6.8 \\
& & S-MoE & 92.1 $\pm$ 0.4 & 75.2 $\pm$ 0.7 & 73.4 $\pm$ 0.6 & 39.1 $\pm$ 0.9 & 70.0 $\pm$ 0.7 & 85 $\pm$ 6 & 20 $\pm$ 3 & 18 $\pm$ 4 & 22 $\pm$ 3 & 36.3 $\pm$ 4.0 \\
& & SEER-MoE & 92.3 $\pm$ 0.3 & 77.1 $\pm$ 0.5 & 74.6 $\pm$ 0.6 & 40.3 $\pm$ 0.7 & 71.1 $\pm$ 0.5 & 77 $\pm$ 4 & 17 $\pm$ 3 & 17 $\pm$ 2 & 18 $\pm$ 3 & 32.3 $\pm$ 3.0 \\
& & DiEP & 92.9 $\pm$ 0.2 & 77.6 $\pm$ 0.4 & 75.1 $\pm$ 0.5 & 42.1 $\pm$ 0.6 & 71.9 $\pm$ 0.4 & 73 $\pm$ 3 & 15 $\pm$ 2 & 16 $\pm$ 3 & 17 $\pm$ 2 & 30.3 $\pm$ 2.5 \\
& & REAP & 93.1 $\pm$ 0.3 & 78.3 $\pm$ 0.5 & 75.4 $\pm$ 0.6 & 42.6 $\pm$ 0.5 & 72.3 $\pm$ 0.5 & 70 $\pm$ 4 & 14 $\pm$ 2 & 15 $\pm$ 2 & 15 $\pm$ 3 & 28.5 $\pm$ 2.8 \\
& & \cellcolor{gray!30}\textbf{DS-MoE} & \cellcolor{gray!30}\textbf{94.1} $\pm$ \textbf{0.2} & \cellcolor{gray!30}\textbf{81.2} $\pm$ \textbf{0.3} & \cellcolor{gray!30}\textbf{77.9} $\pm$ \textbf{0.2} & \cellcolor{gray!30}\textbf{45.1} $\pm$ \textbf{0.4} & \cellcolor{gray!30}\textbf{74.6} $\pm$ \textbf{0.3} & \cellcolor{gray!30}\textbf{45} $\pm$ \textbf{2} & \cellcolor{gray!30}\textbf{12} $\pm$ \textbf{1} & \cellcolor{gray!30}\textbf{13} $\pm$ \textbf{1} & \cellcolor{gray!30}\textbf{14} $\pm$ \textbf{2} & \cellcolor{gray!30}\textbf{21.0} $\pm$ \textbf{1.5} \\
\bottomrule
\vspace{-0.5em}
\end{tabular}
}

\vspace{-1.1em}
\end{table*}

\textbf{Advanced Multi-Domain Reasoning.} Table~\ref{tab:main_results} highlights DS-MoE's significant advantage on the MMLU benchmark. While SEER-MoE and DiEP marginally improve accuracy over the isolated S-MoE (which yields sub-optimal accuracies of 46.4\%, 46.8\%, 39.1\%, and slow convergence of 19, 16, 22 rounds), they still fail to accurately capture complex expert functional correlation. In contrast, by the design of difference-of-submodular optimization, DS-MoE successfully decouples functional synergy and redundancy across experts. This enables DS-MoE to achieve state-of-the-art accuracies (51.8\%, 52.4\%, 45.1\%) while minimizing convergence to 11, 10, and 14 rounds across the three architectures.

\vspace{-1em}
\subsection{Micro-Benchmarking}
\vspace{-0.7em}

\textbf{Robustness to Expert Sparsity ($k$).} 
Table~\ref{tab:micro_benchmarking}a evaluates model's robustness across varying sparsity constraints ($k \in \{1, 2, 3\}$) on DeepSeek-MoE-16B. Under extreme sparsity ($k=1$), S-MoE, SEER-MoE, and DiEP degrade significantly to 68.3\%, 69.8\%, and 70.5\% average accuracy, respectively. In contrast, DS-MoE preserves critical expert synergies to achieve a superior 73.3\% average accuracy. As the routing budget expands ($k=3$), DS-MoE efficiently translates the added capacity into a peak 74.9\% accuracy. Beyond accuracy, DS-MoE consistently accelerates convergence over all baselines. At $k = 3$, it reduces the average convergence time to merely 18.3 rounds, achieving approximately 1.5$\times$, 1.7$\times$, and 1.8$\times$ speedups over DiEP (27.8 rounds), SEER-MoE (30.8 rounds), and S-MoE (32.5 rounds), respectively. 
% This underscores its capability to balance synergy and redundancy effectively, regardless of capacity limits. 

\textbf{Scalability Across Expert Pool Sizes ($N$).} 
Table \ref{tab:micro_benchmarking}b evaluates model performance as the expert search space expands from $N=32$ to $64$ at a fixed sparsity ($k = 2$). While an enlarged pool increases model capacity, it triggers an exponential explosion in possible expert combinations, making optimal routing increasingly difficult for standard mechanisms. Despite this, DS-MoE consistently outperforms all baselines, improving the average accuracy from 73.6\% to 74.6\% at $N=64$. Furthermore, DS-MoE maintains stable convergence; at $N=64$, it converges in just 21.0 rounds, preserving a significant efficiency advantage over DiEP (30.3 rounds) and SEER-MoE (32.3 rounds). This confirms DS-MoE’s robust capability to reliably extract synergistic expert combinations at scale.

\begin{figure}[t!]
    \centering
    % ---------------------- Row 1: Hessian vs First-order (Top) ----------------------
    \subfloat[AGNews]{
        \includegraphics[width=0.20\linewidth]{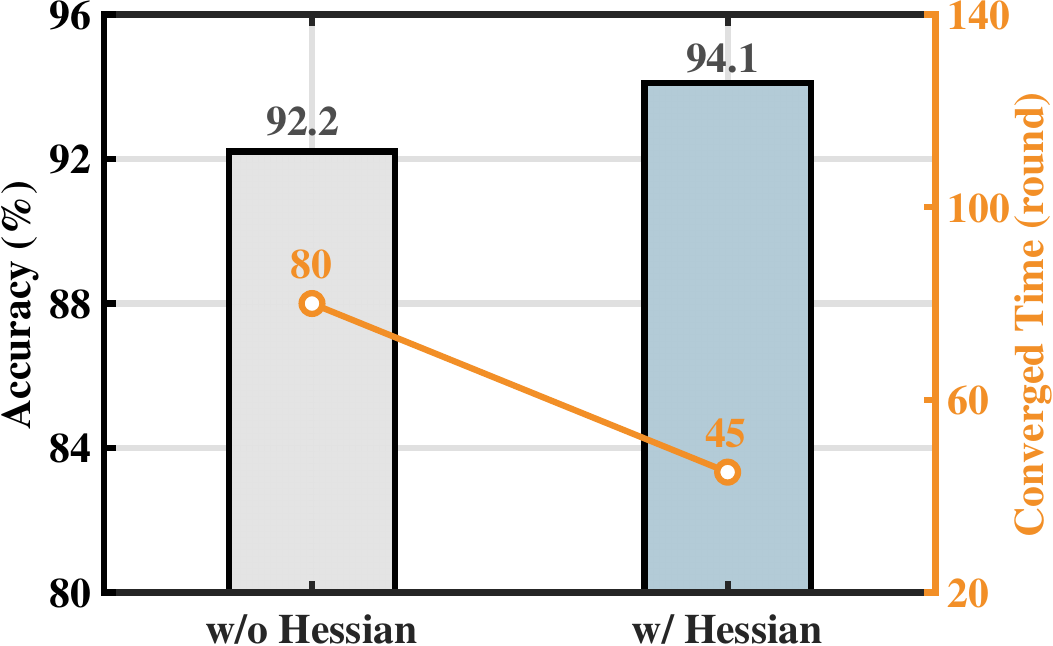}}
    \hfill
    \subfloat[PIQA]{
        \includegraphics[width=0.20\linewidth]{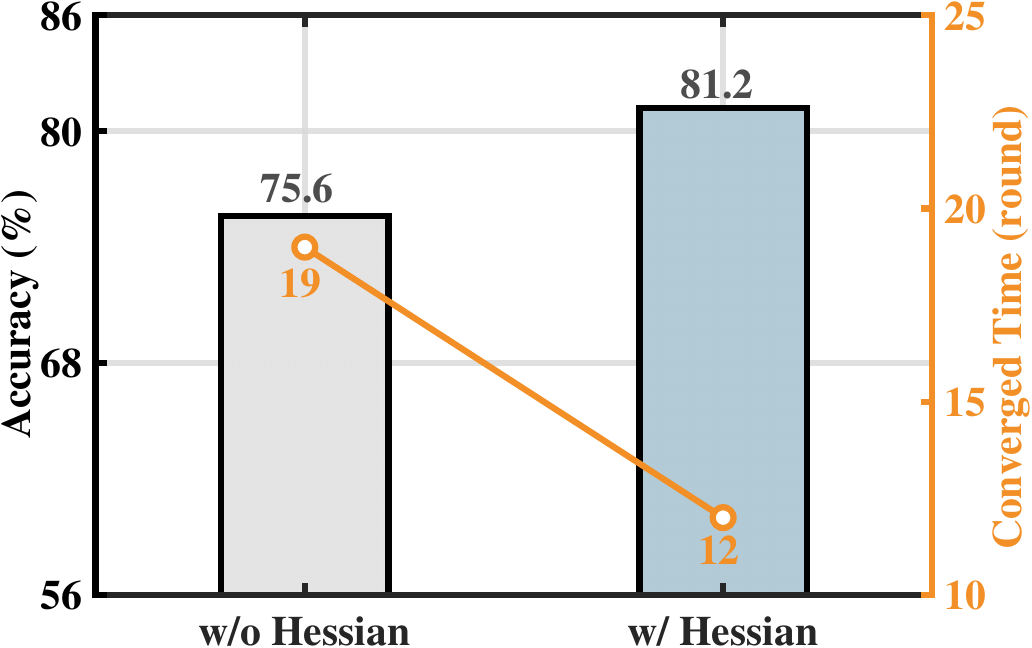}}
    \hfill
    \subfloat[HellaSwag]{
        \includegraphics[width=0.20\linewidth]{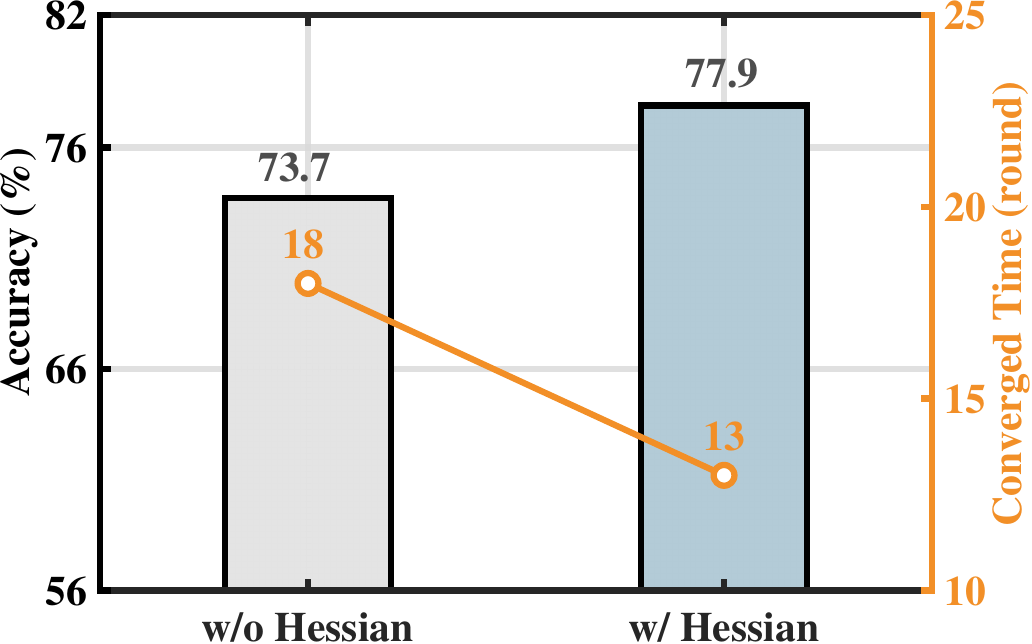}}
    \hfill
    \subfloat[MMLU]{
        \includegraphics[width=0.20\linewidth]{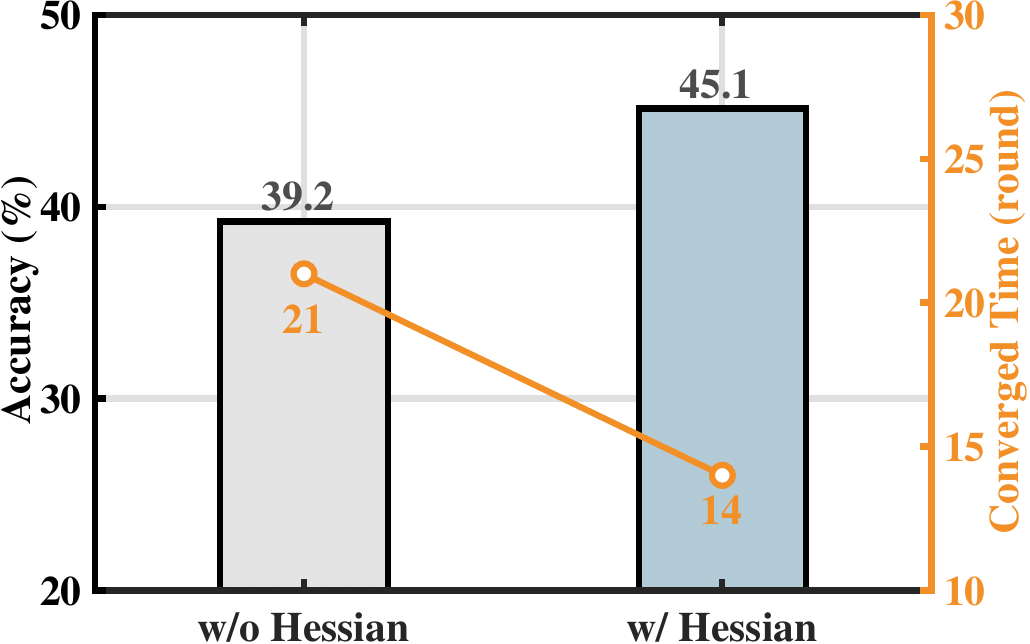}}
    \\[-2.5ex]

    % ---------------------- Row 2: DS Decomposition (Middle) ----------------------
    \subfloat[AGNews]{
        \includegraphics[width=0.20\linewidth]{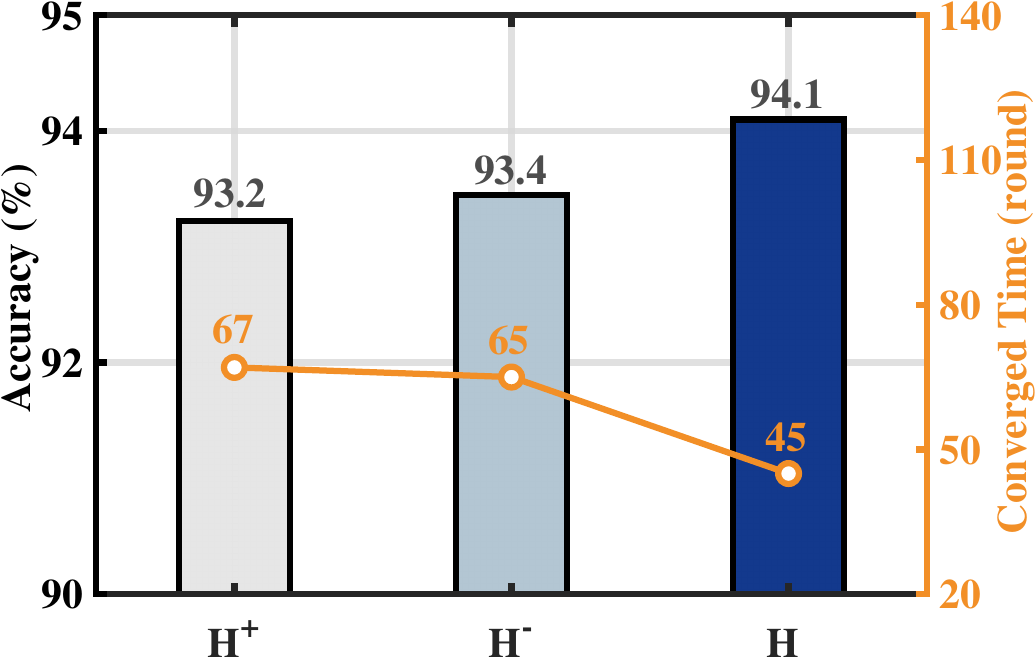}}
    \hfill
    \subfloat[PIQA]{
        \includegraphics[width=0.20\linewidth]{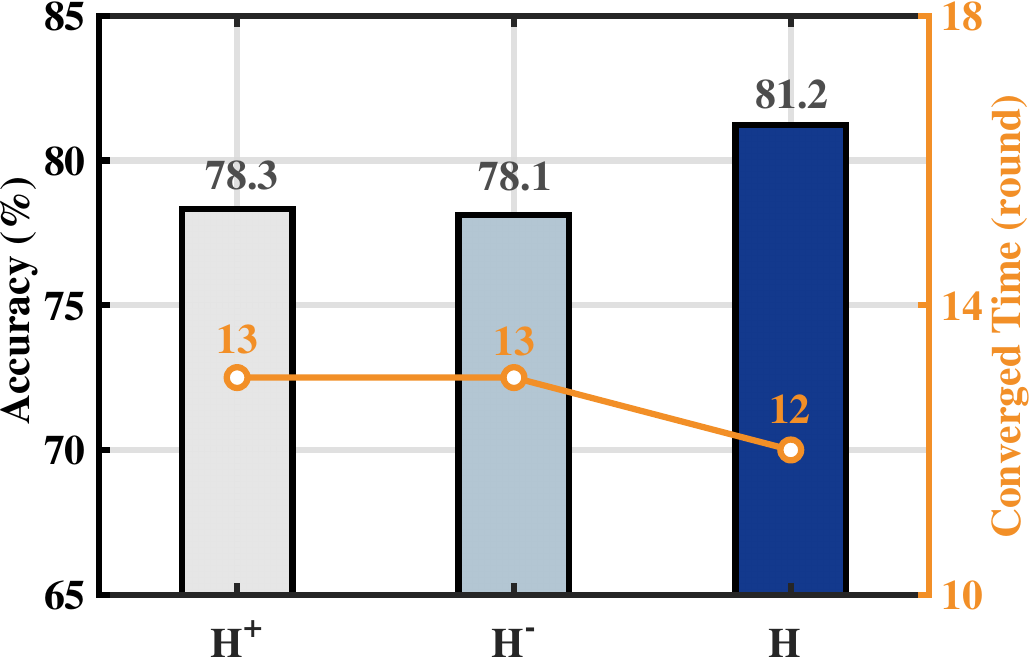}}
    \hfill
    \subfloat[HellaSwag]{
        \includegraphics[width=0.20\linewidth]{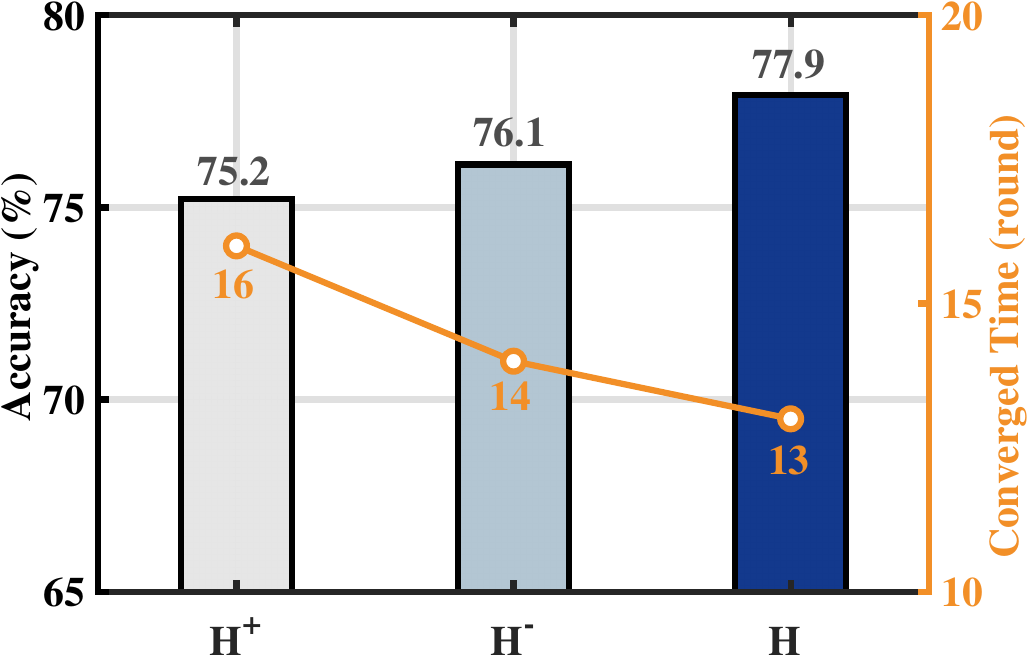}}
    \hfill
    \subfloat[MMLU]{
        \includegraphics[width=0.20\linewidth]{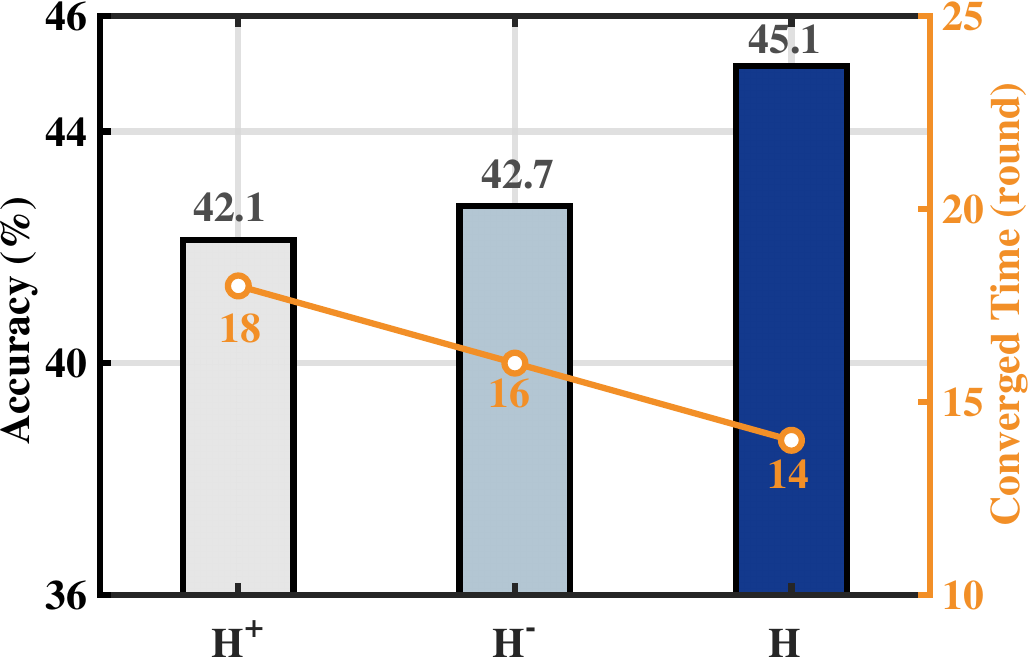}}
    \\[-2.5ex]
    % ---------------------- Row 3: MM Algorithm (Bottom) ----------------------
    \subfloat[AGNews]{
        \includegraphics[width=0.20\linewidth]{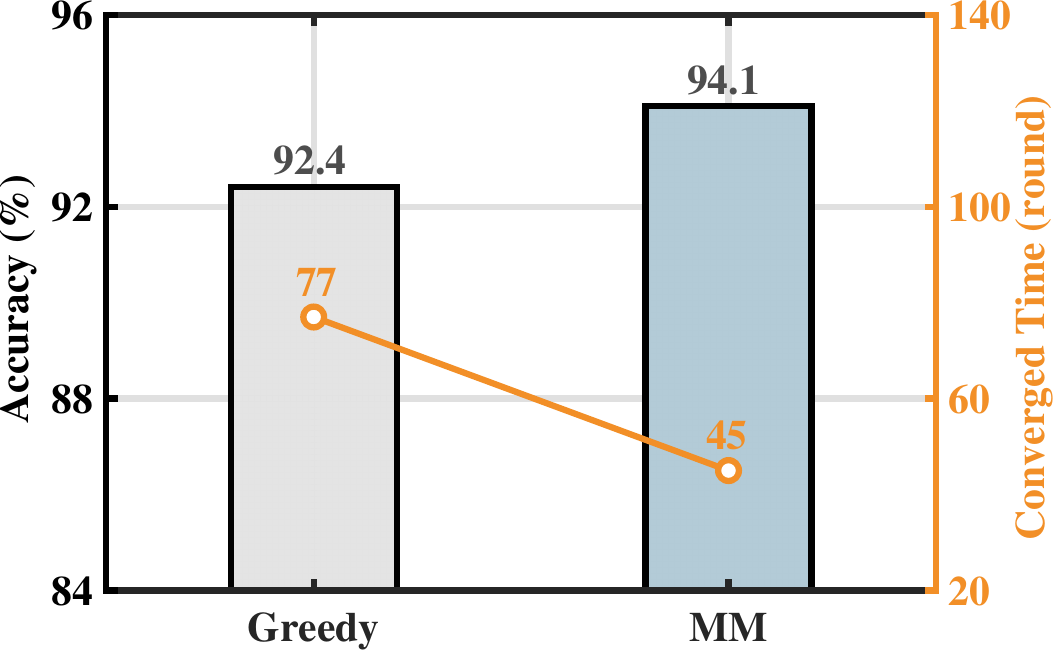}}
    \hfill
    \subfloat[PIQA]{
        \includegraphics[width=0.20\linewidth]{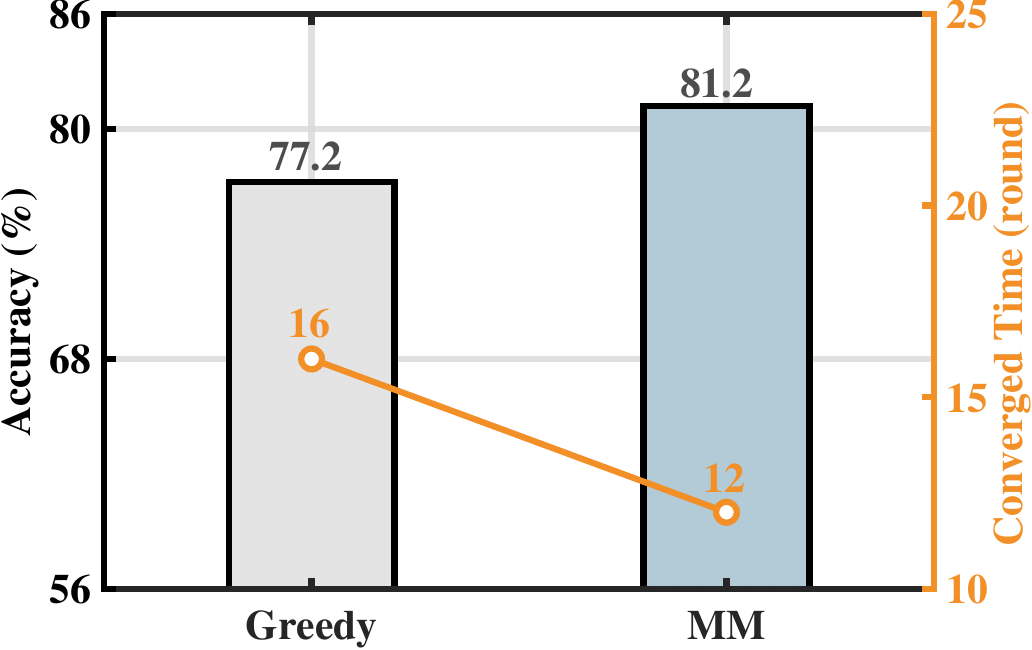}}
    \hfill
    \subfloat[HellaSwag]{
        \includegraphics[width=0.20\linewidth]{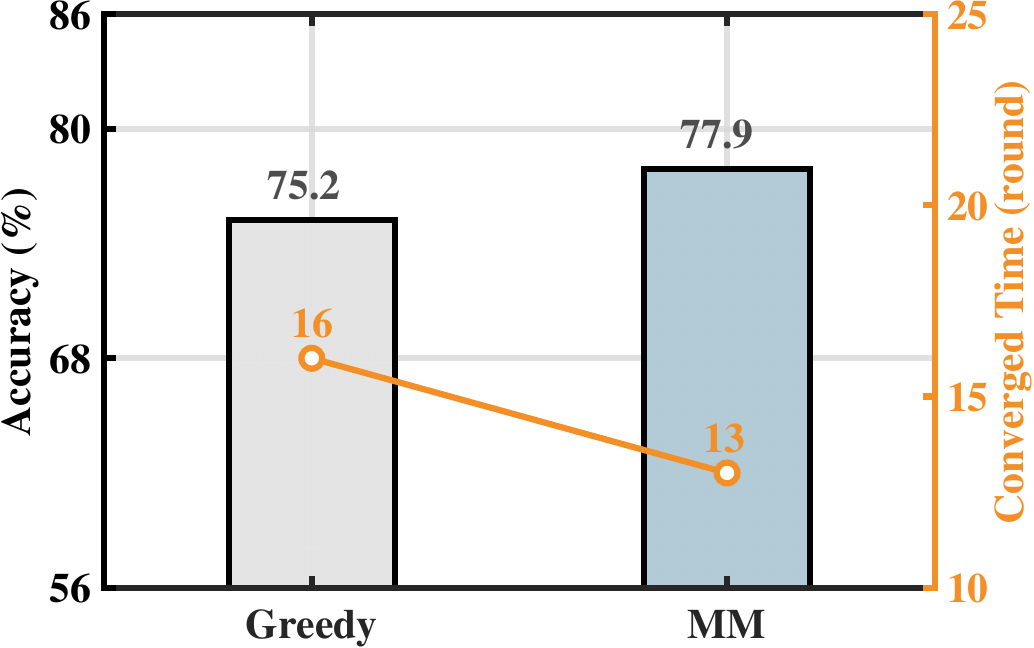}}
    \hfill
    \subfloat[MMLU]{
        \includegraphics[width=0.20\linewidth]{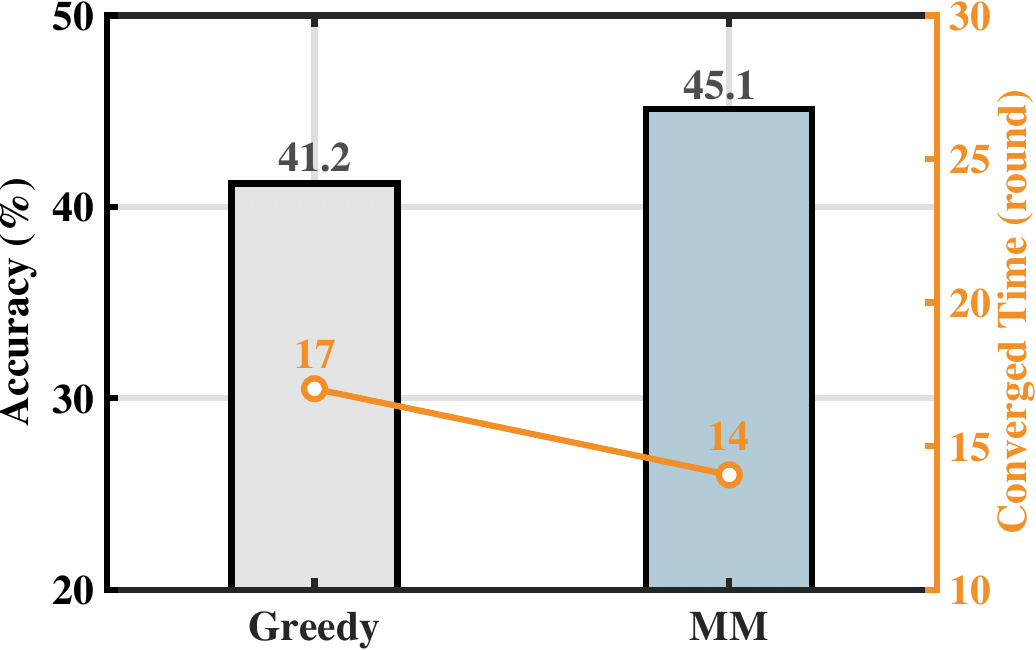}}
    \\
    \caption{{Comprehensive ablation study of the DS-MoE framework on DeepSeek-MoE-16B with $k=2$ and $N=64$. {{(a)-(d)}:} Effectiveness of second-order Hessian over first-order proxies. {{(e)-(h)}:} Necessity of DS decomposition. {{(i)-(l)}:} Optimization superiority of the proposed MM algorithm.}}
    \label{fig:abalations}
    \vspace{-1em}
\end{figure}

% {Comprehensive ablation study of the DS-MoE framework on DeepSeek-MoE-16B with $k=2$ and $N=64$.  {\bf{Top:}} Effectiveness of second-order Hessian over first-order proxies. {\bf{Middle:}} Necessity of DS decomposition. {\bf{Bottom:}} Optimization superiority of the proposed MM algorithm.
%   }
\subsection{Ablation Study}
{\bf{Study on the Impact of Hessian Matrix.}} Figures~\ref{fig:abalations}(a)-\ref{fig:abalations}(d) compare a purely first-order gradient proxy (w/o Hessian) against our framework utilizing the second-order Hessian matrix (w/ Hessian). The results demonstrate that integrating the Hessian matrix consistently yields higher test accuracy and a sharp reduction in convergence rounds. This validates a fundamental limitation of first-order heuristics: they fail to capture loss landscape curvature and suffer from vanishing gradients in flat regions. In contrast, formulating expert selection via the second-order Hessian explicitly captures combinatorial inter-expert dependencies, enabling accurate evaluation of expert synergies and circumventing sub-optimal local minima.

{\bf{Study on the Impact of DS Decomposition.}} Figures~\ref{fig:abalations}(e)-\ref{fig:abalations}(h) evaluates the necessity of the DS decomposition by comparing the complete framework (${\bf H}$) against isolated variants employing strictly the redundancy penalty (${\bf H}^+$) or synergy reward (${\bf H}^-$). The results reveal that deploying either component individually yields marginal performance gains, integrating both via the complete DS decomposition consistently achieves the highest test accuracy and fastest convergence. This confirms that mathematically decoupling redundancy reduction from synergy maximization is indispensable for fully exploiting inter-expert dependencies and optimizing routing efficacy.

{\bf{Study on the Impact of MM Algorithm.}} Figures~\ref{fig:abalations}(i)-\ref{fig:abalations}(l) evaluate the effectiveness of the proposed MM algorithm against a standard local greedy search strategy (Greedy). We see that the greedy approach is highly susceptible to sub-optimal local minima, manifesting as degraded test accuracy and prolonged convergence times. Conversely, the MM algorithm consistently achieves superior accuracy while significantly reducing the required training rounds. These empirical gains validate the monotonic convergence guarantees intrinsic to the MM framework. By bounding the highly non-convex objective via tight surrogate functions, the MM algorithm can circumvent the sub-optimal expert selection inherent to heuristic greedy searches.

\vspace{-0.5em}
\section{Conclusion}
\vspace{-0.5em}
In this paper, we introduced DS-MoE, a theoretically grounded MoE selection framework. By leveraging the Hessian of the loss degradation, we quantify complex inter-expert dependencies, revealing fundamental duality between functional redundancy and cooperative synergy. To effectively harness this duality, we formulate the expert selection process as a DS optimization problem and develop a tailored MM algorithm with a provable convergence guarantee. 
% Extensive empirical evaluations demonstrate that DS-MoE achieves better model performance than state-of-the-art baselines. 
Finally, we discuss the limitations of DS-MoE and outline promising directions for future research in Appendix~\ref{sec:limitations}, and present the practical impact of DS-MoE in Appendix~\ref{sec:practical_impact}.

\bibliography{iclr2027_conference}
\bibliographystyle{iclr2027_conference}

% \clearpage
\appendix
\section{Appendix}

\addtocontents{toc}{\protect\setcounter{tocdepth}{2}}

% \begin{center}
%     {\LARGE \textbf{Appendix}}
% \end{center}
% \vspace{1em}

% 2. 生成目录：使用原生 tableofcontents，确保排版风格与 NeurIPS 一致
{
    \hypersetup{linkcolor=black} % 将目录链接设为黑色，避免红框/蓝字影响美观
    \renewcommand{\contentsname}{Contents} % 将目录标题改为 Table of Contents
    \tableofcontents
}

\vspace{2em}
\hrule height 1pt
\vspace{2em}

% --- A. Extended Related Work ---
\section{Related Work and Research Gap} \label{sec:related_work}

% \subsection{Advances in MoE Architecture and Routing} \label{sec:moe_routing}

The MoE architecture has emerged as a pivotal paradigm for scaling model capacity without incurring proportional computational overhead via sparse expert activation. Recent advancements in MoE primarily focus on routing optimizations and structural refinements:

For {\textit{routing optimizations}}, existing methods are categorized into token-level~\cite{shazeer2017outrageously,fedus2022switch,lepikhin2020gshard,jiang2024mixtral} and expert-level routing~\cite{zhou2022mixture,puigcerver2024sparse,lin2024scattermoe}. Token-level routing, pioneered by Sparsely-Gated MoE~\cite{shazeer2017outrageously}, employs a Top-$k$ mechanism to dispatch each token to an expert combination based on localized routing probabilities; Switch Transformer~\cite{fedus2022switch} streamlined this mechanism by routing each token to only the Top-1 expert, thereby improving training efficiency; GShard~\cite{lepikhin2020gshard} introduced automatic sharding to enable conditional computation in LLMs; Mixtral~\cite{jiang2024mixtral} demonstrated the scalability of a sparse Top-2 routing strategy within open-weight foundation models. For expert-level routing, expert choice routing~\cite{zhou2022mixture} inverted this logic by allowing experts to selectively pull their top-scoring tokens up to a predefined capacity for enforcing strict load balancing; Soft-MoE~\cite{puigcerver2024sparse} addressed discrete routing instabilities by allowing experts to pull continuous, weighted combinations of tokens; ScatterMoE~\cite{lin2024scattermoe} optimized the expert choice mechanism by eliminating padding and redundant memory accesses during token aggregation.

For {\textit{structural refinements}},  several frameworks have been proposed to enhance system execution efficiency:~\cite{he2021fastmoe,gale2023megablocks,zoph2022stmoe,muzio2024seer,dai2024deepseekmoe}: FastMoE~\cite{he2021fastmoe} accelerated training by optimizing communication primitives and customized routing operators; MegaBlocks~\cite{gale2023megablocks} utilized block-sparse matrix operations to support variable-capacity routing without token-dropping overhead; ST-MoE~\cite{zoph2022stmoe} incorporated auxiliary router z-loss to stabilize routing dynamics; SEER-MoE~\cite{muzio2024seer}  introduced differentiable sparsity regularizers to refine gating distributions; and DeepSeek-MoE~\cite{dai2024deepseekmoe} advocated for finer-grained expert partitioning combined with isolated shared experts to capture generalized knowledge.

% \subsection{The Lack of Quantification for Expert Interdependence} \label{sec:expert_selection}

While the aforementioned architectural and routing innovations~\cite{shazeer2017outrageously,fedus2022switch,lepikhin2020gshard,jiang2024mixtral,zhou2022mixture,puigcerver2024sparse,lin2024scattermoe,he2021fastmoe,gale2023megablocks,zoph2022stmoe,muzio2024seer,dai2024deepseekmoe} significantly improve MoE efficiency, they share a fundamental limitation: experts are evaluated and selected in absolute isolation. Whether governed by localized gating probabilities or global sparsity regularizers, existing frameworks~\cite{zhou2022mixture,puigcerver2024sparse,lin2024scattermoe} compute selection scores independently for each individual expert. Consequently, these approaches inherently fail to distinguish between functional redundancy (i.e., experts encoding overlapping representation subspaces) and synergy (i.e., specific expert combinations providing complementary error compensation). To bridge this gap and explicitly model inter-expert relationships, a rigorous analytical framework is essential. In the broader context of neural network compression, quantifying such interdependencies via second-order (Hessian) information has a well-established theoretical foundation~\cite{lecun1989optimal,hassibi1992second}. However, adapting Hessian-aware analysis to effectively quantify the complicated combinatorial dependencies within MoE architectures remains a critical, yet unexplored frontier.

% --- B. Mathematical Proofs and Theoretical Derivations ---
\section{Mathematical Proofs and Theoretical Derivations} \label{sec:proofs}

\subsection{Derivation of the Retained Utility Function} \label{sec:utility_derivation}

Recalling the second-order Taylor expansion of the loss degradation around the full model $\mathbf{m}=\mathbf{1}$ from Eqn.~\ref{eqn:loss_expansion}, we have:
\begin{equation}\label{eqn:loss_expansion_appendix}
    \Delta\mathcal{L}(\mathbf{m}) \approx (\mathbf{m}-\mathbf{1})^{\top}\nabla_{\mathbf{m}}\mathcal{L}|_{\mathbf{m}=\mathbf{1}} + \frac{1}{2}(\mathbf{m}-\mathbf{1})^{\top}\mathbf{H}_{\mathbf{m}}|_{\mathbf{m}=\mathbf{1}}(\mathbf{m}-\mathbf{1}).
\end{equation}

Our optimization objective is to minimize this expected loss degradation, which is mathematically equivalent to maximizing the retained utility function $\mathcal{U}(\mathbf{m}) = -\Delta\mathcal{L}(\mathbf{m})$. With a slight abuse of notation, we denote the right-hand side of Eqn.~\ref{eqn:loss_expansion_appendix} by $\Delta\mathcal{L}(\mathbf{m})$ hereafter. To convert this continuous approximation into a tractable cardinality-constrained QBO problem, we expand the vector-matrix multiplications into their scalar summation forms:
\begin{equation}\label{eqn:scalar_summation}
    \Delta\mathcal{L}(\mathbf{m}) = \sum_{i=1}^{N}(m_{i}-1)[\nabla_{\mathbf{m}}\mathcal{L}|_{\mathbf{m}=\mathbf{1}}]_{i} + \frac{1}{2}\sum_{i=1}^{N}\sum_{j=1}^{N}(m_{i}-1)(m_{j}-1)[\mathbf{H}_{\mathbf{m}}|_{\mathbf{m}=\mathbf{1}}]_{ij}.
\end{equation}

By expanding the polynomial product $(m_{i}-1)(m_{j}-1) = m_{i}m_{j} - m_{i} - m_{j} + 1$ within the double summation of Eqn.~\ref{eqn:scalar_summation}, we can isolate the quadratic, linear, and constant components:
\begin{equation}\label{eqn:expend_scalar_summation}
    \Delta\mathcal{L}(\mathbf{m}) = \sum_{i=1}^{N}(m_{i}[\nabla_{\mathbf{m}}\mathcal{L}|_{\mathbf{m}=\mathbf{1}}]_{i} - [\nabla_{\mathbf{m}}\mathcal{L}|_{\mathbf{m}=\mathbf{1}}]_{i}) + \frac{1}{2}\sum_{i=1}^{N}\sum_{j=1}^{N}(m_{i}m_{j} - m_{i} - m_{j} + 1)[\mathbf{H}_{\mathbf{m}}|_{\mathbf{m}=\mathbf{1}}]_{ij}.
\end{equation}

Distributing the summations in Eqn.~\ref{eqn:expend_scalar_summation}, we obtain:
\begin{equation}\label{eqn:full_expend}
    \begin{aligned}
        \Delta\mathcal{L}(\mathbf{m}) = &\sum_{i=1}^{N}m_{i}[\nabla_{\mathbf{m}}\mathcal{L}|_{\mathbf{m}=\mathbf{1}}]_{i} - \sum_{i=1}^{N}[\nabla_{\mathbf{m}}\mathcal{L}|_{\mathbf{m}=\mathbf{1}}]_{i} + \frac{1}{2}\sum_{i=1}^{N}\sum_{j=1}^{N}m_{i}m_{j}[\mathbf{H}_{\mathbf{m}}|_{\mathbf{m}=\mathbf{1}}]_{ij} \\
        &- \frac{1}{2}\sum_{i=1}^{N}\sum_{j=1}^{N}m_{i}[\mathbf{H}_{\mathbf{m}}|_{\mathbf{m}=\mathbf{1}}]_{ij} - \frac{1}{2}\sum_{i=1}^{N}\sum_{j=1}^{N}m_{j}[\mathbf{H}_{\mathbf{m}}|_{\mathbf{m}=\mathbf{1}}]_{ij} + \frac{1}{2}\sum_{i=1}^{N}\sum_{j=1}^{N}[\mathbf{H}_{\mathbf{m}}|_{\mathbf{m}=\mathbf{1}}]_{ij}.
    \end{aligned}
\end{equation}

Notably, the empirical Fisher in Eqn.~\ref{eqn:appro_H_m} is symmetric by construction, implying $[\mathbf{H}_{\mathbf{m}}|_{\mathbf{m}=\mathbf{1}}]_{ij} = [\mathbf{H}_{\mathbf{m}}|_{\mathbf{m}=\mathbf{1}}]_{ji}$ for all $i, j \in \{1, \dots, N\}$. Consequently, the summation over the rows is mathematically identical to the summation over the columns:
\begin{equation}\label{eqn:symmetry_property}
\begin{aligned}
    \sum_{i=1}^{N}\sum_{j=1}^{N}m_{j}[\mathbf{H}_{\mathbf{m}}|_{\mathbf{m}=\mathbf{1}}]_{ij} &= \sum_{j=1}^{N}m_{j}\left(\sum_{i=1}^{N}[\mathbf{H}_{\mathbf{m}}|_{\mathbf{m}=\mathbf{1}}]_{ij}\right) = \sum_{j=1}^{N}m_{j}\left(\sum_{i=1}^{N}[\mathbf{H}_{\mathbf{m}}|_{\mathbf{m}=\mathbf{1}}]_{ji}\right) \\&= \sum_{i=1}^{N}\sum_{j=1}^{N}m_{i}[\mathbf{H}_{\mathbf{m}}|_{\mathbf{m}=\mathbf{1}}]_{ij}.
    \end{aligned}
\end{equation}

Substituting the symmetry property from Eqn.~\ref{eqn:symmetry_property} back into the expanded loss degradation in Eqn.~\ref{eqn:full_expend} helps consolidate the linear terms originating from the quadratic expansion:
\begin{equation}\label{eqn:transformed_expansion}
    \begin{aligned}
        \Delta\mathcal{L}(\mathbf{m}) = &\sum_{i=1}^{N}m_{i}[\nabla_{\mathbf{m}}\mathcal{L}|_{\mathbf{m}=\mathbf{1}}]_{i} - \sum_{i=1}^{N}[\nabla_{\mathbf{m}}\mathcal{L}|_{\mathbf{m}=\mathbf{1}}]_{i} + \frac{1}{2}\sum_{i=1}^{N}\sum_{j=1}^{N}m_{i}m_{j}[\mathbf{H}_{\mathbf{m}}|_{\mathbf{m}=\mathbf{1}}]_{ij} \\
        &- \sum_{i=1}^{N}\sum_{j=1}^{N}m_{i}[\mathbf{H}_{\mathbf{m}}|_{\mathbf{m}=\mathbf{1}}]_{ij} + \frac{1}{2}\sum_{i=1}^{N}\sum_{j=1}^{N}[\mathbf{H}_{\mathbf{m}}|_{\mathbf{m}=\mathbf{1}}]_{ij}.
    \end{aligned}
\end{equation}

Then, we factor out the binary variable $m_{i}$ from all linear terms in Eqn.~\ref{eqn:transformed_expansion} to segregate the terms dependent on $\mathbf{m}$ from the constant terms:
\begin{equation}\label{eqn:transformed_expansion_1}
    \begin{aligned}
        \Delta\mathcal{L}(\mathbf{m}) = &\sum_{i=1}^{N}m_{i}\left([\nabla_{\mathbf{m}}\mathcal{L}|_{\mathbf{m}=\mathbf{1}}]_{i} - \sum_{j=1}^{N}[\mathbf{H}_{\mathbf{m}}|_{\mathbf{m}=\mathbf{1}}]_{ij}\right) + \frac{1}{2}\sum_{i=1}^{N}\sum_{j=1}^{N}m_{i}m_{j}[\mathbf{H}_{\mathbf{m}}|_{\mathbf{m}=\mathbf{1}}]_{ij} \\
        &+ \underbrace{\left(\frac{1}{2}\sum_{i=1}^{N}\sum_{j=1}^{N}[\mathbf{H}_{\mathbf{m}}|_{\mathbf{m}=\mathbf{1}}]_{ij} - \sum_{i=1}^{N}[\nabla_{\mathbf{m}}\mathcal{L}|_{\mathbf{m}=\mathbf{1}}]_{i}\right)}_{\text{Constant } C}.
    \end{aligned}
\end{equation}

We then substitute the consolidated form Eqn.~\ref{eqn:transformed_expansion_1} back into the retained utility function $\mathcal{U}(\mathbf{m}) = -\Delta\mathcal{L}(\mathbf{m})$, reversing the signs:
\begin{equation}\label{eqn:transformed_retain_function}
    \mathcal{U}(\mathbf{m}) = \sum_{i=1}^{N}m_{i}\left(\sum_{j=1}^{N}[\mathbf{H}_{\mathbf{m}}|_{\mathbf{m}=\mathbf{1}}]_{ij} - [\nabla_{\mathbf{m}}\mathcal{L}|_{\mathbf{m}=\mathbf{1}}]_{i}\right) - \frac{1}{2}\sum_{i=1}^{N}\sum_{j=1}^{N}m_{i}m_{j}[\mathbf{H}_{\mathbf{m}}|_{\mathbf{m}=\mathbf{1}}]_{ij} - C.
\end{equation}

Since $C$ is a constant independent of the decision vector $\mathbf{m}$, shifting the objective function by this scalar does not alter the optimal $\mathbf{m}^{*}$ of the maximization problem. Thus, it can be safely discarded.

Finally, we define the individual utility $d_{i}$ to encapsulate the standalone contribution of $E_i$ toward minimizing the loss ($- {[\nabla_{\bf m}\mathcal{L}|_{\mathbf{m}=\mathbf{1}}]_i}$) and the aggregated pairwise coupling effects between $E_i$ and all other experts (${\sum_{j=1}^N [\mathbf{H}_{\bf m}|_{\mathbf{m}=\mathbf{1}}]_{ij}}$):
\begin{equation}\label{eqn:d_i_experssion}
    d_{i} = \sum_{j=1}^{N}[\mathbf{H}_{\mathbf{m}}|_{\mathbf{m}=\mathbf{1}}]_{ij} - [\nabla_{\mathbf{m}}\mathcal{L}|_{\mathbf{m}=\mathbf{1}}]_{i}.
\end{equation}

Substituting Eqn.~\ref{eqn:d_i_experssion} into Eqn.~\ref{eqn:transformed_retain_function}, we obtain the analytical form for the retained utility presented in problem~\eqref{problem:utiliy}:
\begin{equation}
    \mathcal{U}(\mathbf{m}) = \sum_{i=1}^{N}m_{i}d_{i} - \frac{1}{2}\sum_{i=1}^{N}\sum_{j=1}^{N}m_{i}m_{j}[\mathbf{H}_{\mathbf{m}}|_{\mathbf{m}=\mathbf{1}}]_{ij}.
\end{equation}

\subsection{Proof of NP-Hardness for Cardinality-Constrained QBO} \label{sec:np_hard}

We prove the NP-hardness of the considered problem by demonstrating a polynomial-time reduction from the $k$-Clique decision problem, one of Karp's 21 NP-complete problems~\citep{karp1972reducibility}.

Consider an arbitrary instance of the $k$-Clique problem defined on an undirected and unweighted graph $G = (V, E_G)$, where $V$ is the set of vertices with $|V| = N$, and $E_G \subseteq V \times V$ is the set of edges. The structure of $G$ can be algebraically represented by its symmetric adjacency matrix $\mathbf{A} \in \{0, 1\}^{N \times N}$, defined as:
\begin{equation}
    \mathbf{A}_{ij} = 
    \begin{cases} 
    1, & \text{if } (i,j) \in E_G; \\ 
    0, & \text{otherwise}. 
    \end{cases}
\end{equation}
Given a positive integer $k \le N$, the decision problem asks whether there exists a fully connected subgraph (a clique) in $G$ composed of exactly $k$ vertices.

We construct a specific instance of our expert selection problem to mathematically mirror this graph $G$. Let the total number of candidate experts be equal to the number of vertices, $N = |V|$, and set the cardinality constraint to $k$. We define the binary selection vector $\mathbf{m} \in \{0, 1\}^N$, where the scalar $m_i = 1$ indicates that vertex $i$ is selected into our target induced subgraph $V_{\mathbf{m}} \subseteq V$. This establishes the mapping:
\begin{equation}
    V_{\mathbf{m}} = \{i \in V \mid m_i = 1\}, \quad \text{s.t.} \quad |V_{\mathbf{m}}| = \mathbf{1}^\top\mathbf{m} = k.
\end{equation}

To evaluate the retained utility function $\mathcal{U}(\mathbf{m})$ derived in problem~\eqref{problem:utiliy}, we decompose its quadratic penalty into diagonal and off-diagonal components:
\begin{equation}
    \mathcal{U}(\mathbf{m}) = \sum_{i=1}^N m_i d_i - \left( \frac{1}{2} \sum_{i=1}^N m_i^2 [\mathbf{H}_{\mathbf{m}}|_{\mathbf{m}=\mathbf{1}}]_{ii} + \sum_{i < j} m_i m_j [\mathbf{H}_{\mathbf{m}}|_{\mathbf{m}=\mathbf{1}}]_{ij} \right).
\end{equation}

Since the decision variables are binary ($m_i \in \{0, 1\}$), the property $m_i^2 = m_i$ holds. Factoring in this allows us to consolidate all linear terms:
\begin{equation}
\label{eqn:consolidated_utility}
    \mathcal{U}(\mathbf{m}) = \sum_{i=1}^N m_i \left( d_i - \frac{1}{2} [\mathbf{H}_{\mathbf{m}}|_{\mathbf{m}=\mathbf{1}}]_{ii} \right) - \sum_{i < j} m_i m_j [\mathbf{H}_{\mathbf{m}}|_{\mathbf{m}=\mathbf{1}}]_{ij}.
\end{equation}

To isolate the combinatorial pairwise interactions, we first construct our problem instance such that the consolidated linear coefficient vanishes, requiring $d_i = \frac{1}{2} [\mathbf{H}_{\mathbf{m}}|_{\mathbf{m}=\mathbf{1}}]_{ii}$. Recalling the definition $d_{i}=\sum_{j=1}^{N}[\mathbf{H}_{\mathbf{m}}|_{\mathbf{m}=\mathbf{1}}]_{ij}-[\nabla_{\mathbf{m}}\mathcal{L}|_{\mathbf{m}=\mathbf{1}}]_{i}$ from Eqn.~\ref{eqn:d_i_experssion}, this is achieved by specifying the initial first-order gradient vector as
\begin{equation}\label{eq:np_grad}
    [\nabla_{\mathbf{m}} \mathcal{L}|_{\mathbf{m}=\mathbf{1}}]_i = \sum_{j=1}^N [\mathbf{H}_{\mathbf{m}}|_{\mathbf{m}=\mathbf{1}}]_{ij} - \frac{1}{2} [\mathbf{H}_{\mathbf{m}}|_{\mathbf{m}=\mathbf{1}}]_{ii} \quad \forall i \in \{1, \dots, N\}.
\end{equation}

Simultaneously, we construct the off-diagonal entries of the Hessian matrix $\mathbf{H}_{\mathbf{m}}|_{\mathbf{m}=\mathbf{1}}$ to encode the negative adjacency matrix of $G$:
\begin{equation}
    [\mathbf{H}_{\mathbf{m}}|_{\mathbf{m}=\mathbf{1}}]_{ij} = -\mathbf{A}_{ij}, \quad \forall i, j \in \{1, \dots, N\}, i \neq j.
\end{equation}

The diagonal entries of $\mathbf{H}_{\mathbf{m}}|_{\mathbf{m}=\mathbf{1}}$ do not affect the reduction because Eqn.~\ref{eq:np_grad} absorbs them into the gradient, so we set $[\mathbf{H}_{\mathbf{m}}|_{\mathbf{m}=\mathbf{1}}]_{ii} = N$. Since each row contains at most $N-1$ off-diagonal entries equal to $-1$, the constructed matrix is strictly diagonally dominant and hence positive definite. Moreover, scaling $\mathbf{H}_{\mathbf{m}}|_{\mathbf{m}=\mathbf{1}}$ and $\nabla_{\mathbf{m}}\mathcal{L}|_{\mathbf{m}=\mathbf{1}}$ by a sufficiently small constant $c > 0$ ensures $\mathbf{H}_{\mathbf{m}}|_{\mathbf{m}=\mathbf{1}} \succeq \nabla_{\mathbf{m}}\mathcal{L}|_{\mathbf{m}=\mathbf{1}} \nabla_{\mathbf{m}}\mathcal{L}|_{\mathbf{m}=\mathbf{1}}^\top$, so that the instance is realizable as the mean and second moment of per-sample gradients in Eqn.~\ref{eqn:appro_H_m}. This scaling multiplies $\mathcal{U}(\mathbf{m})$ by $c$ and thus preserves its maximizers; for clarity, we present the reduction with $c = 1$.

Constructing these matrices and gradient mappings requires $\mathcal{O}(N^2)$ operations, bounding the computational transformation within polynomial time. By substituting our exact constructions into the consolidated utility equation, the linear terms sum exactly to zero. The objective reduces to the summation of off-diagonal entries, i.e., counting the edges within the selected subgraph:
\begin{equation}\label{eq:subgraph}
    \mathcal{U}(\mathbf{m}) = - \sum_{i < j} m_i m_j (-\mathbf{A}_{ij}) = \sum_{i < j} m_i m_j \mathbf{A}_{ij} = |E_G(V_{\mathbf{m}})|.
\end{equation}

Since the induced subgraph contains exactly $k$ vertices, it has at most $\binom{k}{2}$ edges. Thus, the optimization objective in Eqn.~\ref{eq:subgraph} is upper-bounded by:
\begin{equation}
    \max_{\mathbf{m} : \mathbf{1}^\top\mathbf{m} = k} \mathcal{U}(\mathbf{m}) = \max_{|V_{\mathbf{m}}| = k} |E_G(V_{\mathbf{m}})| \le \binom{k}{2} = \frac{k(k-1)}{2}.
\end{equation}

This upper bound is attained if and only if the induced subgraph $V_{\mathbf{m}}$ is a complete graph (i.e., a $k$-clique). Consequently, the formulated QBO problem yields the optimal utility $\mathcal{U}(\mathbf{m}^*) = \frac{k(k-1)}{2}$ if and only if the original graph $G$ contains a $k$-clique. Given that the $k$-clique decision problem is NP-complete, any algorithmic problem capable of determining its solution must be at least as difficult. Because finding the exact solution to our cardinality-constrained QBO problem~\eqref{problem:utiliy} is equivalent to solving the $k$-clique problem, we therefore conclude that problem~\eqref{problem:utiliy} is NP-hard.

\subsection{Proof of Supermodularity for Decomposed Functions} \label{sec:supermodular}

To establish the supermodularity of the decomposed functions $g(\mathbf{m})$ and $f(\mathbf{m})$, we analyze their discrete marginal gains directly on the Boolean hypercube. We first formalize the condition for supermodularity within this discrete space as follows:

\textbf{Definition 1 (Supermodularity on Boolean Hypercube).} A pseudo-Boolean function $h: \{0,1\}^N \rightarrow \mathbb{R}$ is supermodular if it exhibits non-decreasing marginal returns. Formally, let $\mathbf{m}^A, \mathbf{m}^B \in \{0,1\}^N$ be two binary expert selection vectors such that $\mathbf{m}^A \le \mathbf{m}^B$ (i.e., $m^A_i \le m^B_i$ for all $i \in \{1, \dots, N\}$, implying the set of active experts in $\mathbf{m}^A$ is a subset of that in $\mathbf{m}^B$). For any inactive expert $l$ where $m^A_l = m^B_l = 0$, let $\mathbf{e}_l \in \{0,1\}^N$ denote the one-hot indicator vector for expert $l$. The supermodularity condition requires:
\begin{equation}\label{eq:supermodular_def}
    h(\mathbf{m}^A + \mathbf{e}_l) - h(\mathbf{m}^A) \le h(\mathbf{m}^B + \mathbf{e}_l) - h(\mathbf{m}^B).
\end{equation}

Following Definition 1 in \eqref{eq:supermodular_def}, we evaluate the discrete marginal gains for both the gain function $f(\mathbf{m})$ and the cost function $g(\mathbf{m})$. In the subsequent analysis, we demonstrate that both functions are supermodular and that $g(\mathbf{m})$ is additionally monotonically non-decreasing.

\textbf{Supermodularity of Cost Function $g(\mathbf{m})$.} Recall the formulation of the cost function $g(\mathbf{m})$ from Eqn.~\ref{eq:transformed_problem_eqn}:
\begin{equation}\label{eq:cost_def}
    g(\mathbf{m}) = \frac{1}{2} \sum_{i=1}^N \sum_{j=1}^N m_i m_j [\mathbf{H}_{\mathbf{m}}^+|_{\mathbf{m}=\mathbf{1}}]_{ij}.
\end{equation}
By definition of decomposition, the matrix $\mathbf{H}_{\mathbf{m}}^+$ isolates the non-negative components of the Hessian, guaranteeing that $[\mathbf{H}_{\mathbf{m}}^+|_{\mathbf{m}=\mathbf{1}}]_{ij} \ge 0$ for all $i, j$. Since the empirical Fisher in Eqn.~\ref{eqn:appro_H_m} is symmetric by construction, and the element-wise thresholding preserves this structural symmetry, we have $[\mathbf{H}_{\mathbf{m}}^+|_{\mathbf{m}=\mathbf{1}}]_{ij} = [\mathbf{H}_{\mathbf{m}}^+|_{\mathbf{m}=\mathbf{1}}]_{ji}$.

To evaluate the marginal penalty of activating a new expert $l$ (i.e., transitioning from the state $\mathbf{m}^A$ to $\mathbf{m}^A + \mathbf{e}_l$), we compute the discrete difference of the quadratic cost function in Eqn.~\ref{eq:cost_def}. Expanding this formulation with respect to the increment $\mathbf{e}_l$ yields the diagonal entry representing the intrinsic self-redundancy of expert $l$ and the off-diagonal summations that capture its cross-redundancy with the active set in $\mathbf{m}^A$:
\begin{equation}\label{eq:cost_marginal_expansion}
    g(\mathbf{m}^A + \mathbf{e}_l) - g(\mathbf{m}^A) = \frac{1}{2} \left( \sum_{j=1}^N m^A_j [\mathbf{H}_{\mathbf{m}}^+|_{\mathbf{m}=\mathbf{1}}]_{lj} + \sum_{i=1}^N m^A_i [\mathbf{H}_{\mathbf{m}}^+|_{\mathbf{m}=\mathbf{1}}]_{il} + [\mathbf{H}_{\mathbf{m}}^+|_{\mathbf{m}=\mathbf{1}}]_{ll} \right).
\end{equation}

By applying the symmetry property $[\mathbf{H}_{\mathbf{m}}^+|_{\mathbf{m}=\mathbf{1}}]_{lj} = [\mathbf{H}_{\mathbf{m}}^+|_{\mathbf{m}=\mathbf{1}}]_{jl}$ into Eqn.~\ref{eq:cost_marginal_expansion}, the two cross-term summations inside the parentheses are mathematically equivalent. Consolidating these terms yields a simplified expression for the marginal penalty:
\begin{equation}\label{eq:cost_marginal_A}
    g(\mathbf{m}^A + \mathbf{e}_l) - g(\mathbf{m}^A) = \frac{1}{2} [\mathbf{H}_{\mathbf{m}}^+|_{\mathbf{m}=\mathbf{1}}]_{ll} + \sum_{i=1, i \neq l}^N m^A_i [\mathbf{H}_{\mathbf{m}}^+|_{\mathbf{m}=\mathbf{1}}]_{il}.
\end{equation}
Since every term in Eqn.~\ref{eq:cost_marginal_A} is non-negative for any $\mathbf{m}^A$ with $m^A_l = 0$, adding an expert never decreases $g(\mathbf{m})$, i.e., $g(\mathbf{m})$ is monotonically non-decreasing.

Following the exact same algebraic expansion to the superset selection vector $\mathbf{m}^B$ ($\mathbf{m}^A \le \mathbf{m}^B$), we obtain its corresponding marginal penalty:
\begin{equation}\label{eq:cost_marginal_B}
    g(\mathbf{m}^B + \mathbf{e}_l) - g(\mathbf{m}^B) = \frac{1}{2} [\mathbf{H}_{\mathbf{m}}^+|_{\mathbf{m}=\mathbf{1}}]_{ll} + \sum_{i=1, i \neq l}^N m^B_i [\mathbf{H}_{\mathbf{m}}^+|_{\mathbf{m}=\mathbf{1}}]_{il}.
\end{equation}

To verify the supermodularity condition, we calculate the discrete difference between these two marginal gains of $\mathbf{m}^A$ and $\mathbf{m}^B$. Subtracting \eqref{eq:cost_marginal_A} from \eqref{eq:cost_marginal_B} eliminates the constant quadratic self-interaction term $\frac{1}{2} [\mathbf{H}_{\mathbf{m}}^+|_{\mathbf{m}=\mathbf{1}}]_{ll}$, leaving only the difference in the cross-redundancy summations:
\begin{equation}\label{eq:cost_marginal_diff}
    (g(\mathbf{m}^B + \mathbf{e}_l) - g(\mathbf{m}^B)) - (g(\mathbf{m}^A + \mathbf{e}_l) - g(\mathbf{m}^A)) = \sum_{i=1, i \neq l}^N (m^B_i - m^A_i) [\mathbf{H}_{\mathbf{m}}^+|_{\mathbf{m}=\mathbf{1}}]_{il}.
\end{equation}

Because $\mathbf{m}^A$ is a subset of $\mathbf{m}^B$ ($\mathbf{m}^A \le \mathbf{m}^B$), the increment $(m^B_i - m^A_i)$ is non-negative for all experts $i$. Coupled with the inherent non-negativity of the decomposed Hessian entries $([\mathbf{H}_{\mathbf{m}}^+|_{\mathbf{m}=\mathbf{1}}]_{il} \ge 0)$, their product in \eqref{eq:cost_marginal_diff} is non-negative:
\begin{equation}
    \sum_{i=1, i \neq l}^N (m^B_i - m^A_i) [\mathbf{H}_{\mathbf{m}}^+|_{\mathbf{m}=\mathbf{1}}]_{il} \ge 0.
\end{equation}
This non-negativity guarantees that $g(\mathbf{m}^A + \mathbf{e}_l) - g(\mathbf{m}^A) \le g(\mathbf{m}^B + \mathbf{e}_l) - g(\mathbf{m}^B)$, proving that the cost function $g(\mathbf{m})$ is supermodular: adding a new expert to a larger activated expert subset incurs a greater or equal redundancy penalty than adding it to a smaller subset.

\textbf{Supermodularity of Gain Function $f(\mathbf{m})$.} Next, we evaluate the gain function $f(\mathbf{m})$ formulated in Eqn.~\ref{eq:transformed_problem_eqn}, which captures both the standalone utility of individual experts and functional synergy arising from complementary interactions: 
\begin{equation}
    f(\mathbf{m}) = \sum_{i=1}^N m_i d_i + \frac{1}{2} \sum_{i=1}^N \sum_{j=1}^N m_i m_j [\mathbf{H}_{\mathbf{m}}^-|_{\mathbf{m}=\mathbf{1}}]_{ij},
    \label{eq:gain_def}
\end{equation}
where the matrix $\mathbf{H}_{\mathbf{m}}^-$ isolates the negative components of the Hessian and inverts their signs, thereby maintaining non-negativity $([\mathbf{H}_{\mathbf{m}}^-|_{\mathbf{m}=\mathbf{1}}]_{ij} \ge 0)$ while preserving structural symmetry.

To calculate the marginal gain of activating expert $l$ for the activated subset $\mathbf{m}^A$, we expand the function over the increment $\mathbf{e}_l$. The linear summation $\sum m_i d_i$ in \eqref{eq:gain_def} increases by the standalone capability $d_l$. Following the same algebraic expansion used for $g(\mathbf{m})$, the exact marginal gain evaluates to:
\begin{equation}\label{eq:gain_marginal_A}
    f(\mathbf{m}^A + \mathbf{e}_l) - f(\mathbf{m}^A) = d_l + \frac{1}{2} [\mathbf{H}_{\mathbf{m}}^-|_{\mathbf{m}=\mathbf{1}}]_{ll} + \sum_{i=1, i \neq l}^N m^A_i [\mathbf{H}_{\mathbf{m}}^-|_{\mathbf{m}=\mathbf{1}}]_{il}.
\end{equation}

By an analogous expansion for the denser selection vector $\mathbf{m}^B$ ($\mathbf{m}^A \le \mathbf{m}^B$), the corresponding marginal gain is given by:
\begin{equation} \label{eq:gain_marginal_B}
    f(\mathbf{m}^B + \mathbf{e}_l) - f(\mathbf{m}^B) = d_l + \frac{1}{2} [\mathbf{H}_{\mathbf{m}}^-|_{\mathbf{m}=\mathbf{1}}]_{ll} + \sum_{i=1, i \neq l}^N m^B_i [\mathbf{H}_{\mathbf{m}}^-|_{\mathbf{m}=\mathbf{1}}]_{il}.
\end{equation}

To verify the supermodularity condition, we evaluate the discrete difference. Notably, the standalone capability $d_l$ is a pure modular component; its marginal contribution is a fixed constant irrespective of the previously selected experts. Consequently, subtracting \eqref{eq:gain_marginal_A} from \eqref{eq:gain_marginal_B} perfectly cancels both $d_l$ and the quadratic self-interaction term, isolating the difference in the cross-synergy summations:
\begin{equation} \label{eq:gain_marginal_diff}
    (f(\mathbf{m}^B + \mathbf{e}_l) - f(\mathbf{m}^B)) - (f(\mathbf{m}^A + \mathbf{e}_l) - f(\mathbf{m}^A)) = \sum_{i=1, i \neq l}^N (m^B_i - m^A_i) [\mathbf{H}_{\mathbf{m}}^-|_{\mathbf{m}=\mathbf{1}}]_{il}.
\end{equation}

Given $\mathbf{m}^A \le \mathbf{m}^B$ and the non-negativity of the Hessian entries ($[\mathbf{H}_{\mathbf{m}}^-|_{\mathbf{m}=\mathbf{1}}]_{il} \ge 0$) in \eqref{eq:gain_marginal_diff}, this residual summation is non-negative:
\begin{equation}
    \sum_{i=1, i \neq l}^N (m^B_i - m^A_i) [\mathbf{H}_{\mathbf{m}}^-|_{\mathbf{m}=\mathbf{1}}]_{il} \ge 0.
\end{equation}
This confirms that the gain function $f(\mathbf{m})$ is supermodular.

\subsection{Derivation of the Cost Function Upper Bound} \label{sec:cost_upper}

To derive a tight surrogate upper bound $U_g(\mathbf{m}; \mathbf{m}^{(t)})$ for the cost function $g(\mathbf{m})$ (i.e., $\mathbf{m}' = \mathbf{m}^{(t)}$ in Eqn.~\ref{g_m_eq_1}), we begin with its original quadratic formulation from Eqn.~\ref{eq:transformed_problem_eqn}:
\begin{equation}\label{eq:cost_original}
g(\mathbf{m}) = \frac{1}{2}\sum_{i=1}^{N}\sum_{j=1}^{N}m_im_j[\mathbf{H}_{\mathbf{m}}^{+}|_{\mathbf{m}=\mathbf{1}}]_{ij}.
\end{equation}

To decouple the binary cross-terms, we leverage the fundamental algebraic identity $m_im_j = \frac{1}{2}(m_i^2 + m_j^2) - \frac{1}{2}(m_i - m_j)^2$. Given that the expert selection variables are binary ($m_i, m_j \in \{0, 1\}$), they satisfy the Boolean property $m_i^2 = m_i$. Substituting this into the objective facilitates reformulating the cost function into a linear component $L(\mathbf{m})$ and a concave component $C(\mathbf{m})$ via a difference-of-convex decomposition:
\begin{equation}\label{eq:dc_decomp}
g(\mathbf{m}) = \underbrace{\frac{1}{4}\sum_{i=1}^{N}\sum_{j=1}^{N}(m_i + m_j)[\mathbf{H}_{\mathbf{m}}^{+}|_{\mathbf{m}=\mathbf{1}}]_{ij}}_{\text{Linear Component } L(\mathbf{m})} + \underbrace{\left( -\frac{1}{4}\sum_{i=1}^{N}\sum_{j=1}^{N}(m_i - m_j)^2 [\mathbf{H}_{\mathbf{m}}^{+}|_{\mathbf{m}=\mathbf{1}}]_{ij} \right)}_{\text{Concave Component } C(\mathbf{m})}.
\end{equation}

We can separate the summation in the linear component $L(\mathbf{m})$ in Eqn.~\ref{eq:dc_decomp}:
\begin{equation}
L(\mathbf{m}) = \frac{1}{4}\sum_{i=1}^{N}\sum_{j=1}^{N}m_i[\mathbf{H}_{\mathbf{m}}^{+}|_{\mathbf{m}=\mathbf{1}}]_{ij} + \frac{1}{4}\sum_{i=1}^{N}\sum_{j=1}^{N}m_j[\mathbf{H}_{\mathbf{m}}^{+}|_{\mathbf{m}=\mathbf{1}}]_{ij}.
\end{equation}

By swapping the dummy indices $i$ and $j$ in the second term, and leveraging the structural symmetry of the Hessian matrix $[\mathbf{H}_{\mathbf{m}}^{+}|_{\mathbf{m}=\mathbf{1}}]_{ji} = [\mathbf{H}_{\mathbf{m}}^{+}|_{\mathbf{m}=\mathbf{1}}]_{ij}$, the linear component $L({\bf m})$ simplifies to
\begin{equation}\label{eq:linear_consolidated}
L(\mathbf{m}) = \frac{1}{2}\sum_{i=1}^{N}m_i \left( \sum_{j=1}^{N}[\mathbf{H}_{\mathbf{m}}^{+}|_{\mathbf{m}=\mathbf{1}}]_{ij} \right).
\end{equation}

Since $[\mathbf{H}_{\mathbf{m}}^{+}|_{\mathbf{m}=\mathbf{1}}]_{ij} \ge 0$, the residual component $C(\mathbf{m})$ in Eqn.~\ref{eq:dc_decomp} is a non-negative combination of the concave functions $-(m_i - m_j)^2$, and hence concave on $\mathbb{R}^N$. Taking the partial derivative with respect to $m_l$ yields
\begin{align}
\frac{\partial C(\mathbf{m})}{\partial m_l} &= -\frac{1}{4} \left[ \sum_{j=1}^{N} 2(m_l - m_j)[\mathbf{H}_{\mathbf{m}}^{+}|_{\mathbf{m}=\mathbf{1}}]_{lj} + \sum_{i=1}^{N} 2(m_i - m_l)(-1)[\mathbf{H}_{\mathbf{m}}^{+}|_{\mathbf{m}=\mathbf{1}}]_{il} \right] \\
&= -\frac{1}{4} \left[ \sum_{j=1}^{N} 2(m_l - m_j)[\mathbf{H}_{\mathbf{m}}^{+}|_{\mathbf{m}=\mathbf{1}}]_{lj} + \sum_{i=1}^{N} 2(m_l - m_i)[\mathbf{H}_{\mathbf{m}}^{+}|_{\mathbf{m}=\mathbf{1}}]_{li} \right].
\end{align}

By applying the symmetry property $[\mathbf{H}_{\mathbf{m}}^{+}|_{\mathbf{m}=\mathbf{1}}]_{il} = [\mathbf{H}_{\mathbf{m}}^{+}|_{\mathbf{m}=\mathbf{1}}]_{li}$, the two summations merge, yielding the general partial derivative with respect to $m_l$:
\begin{equation}\label{eq:concave_gradient}
\frac{\partial C(\mathbf{m})}{\partial m_l} = - \sum_{j=1}^N [\mathbf{H}_{\mathbf{m}}^{+}|_{\mathbf{m}=\mathbf{1}}]_{lj} (m_l - m_j).
\end{equation}

Since a concave function lies below its tangent plane, $C(\mathbf{m})$ is globally upper-bounded by its first-order Taylor expansion at $\mathbf{m}^{(t)}$:
\begin{equation}\label{eq:expansion_talor}
C(\mathbf{m}) \le C(\mathbf{m}^{(t)}) + \sum_{i=1}^{N} [\nabla_{\mathbf{m}} C(\mathbf{m}^{(t)})]_i (m_i - m_i^{(t)}).
\end{equation}

Substituting the gradient from Eqn.~\ref{eq:concave_gradient} evaluated at $\mathbf{m}^{(t)}$ into Eqn.~\ref{eq:expansion_talor} and extracting the terms dependent on $m_i$, we obtain:
\begin{equation}\label{eq:taylor_substituted}
C(\mathbf{m}) \le \sum_{i=1}^{N} m_i \left[ -\sum_{j=1}^{N} [\mathbf{H}_{\mathbf{m}}^{+}|_{\mathbf{m}=\mathbf{1}}]_{ij} (m_i^{(t)} - m_j^{(t)}) \right] + C_c^{(t)},
\end{equation}
where the state-dependent constant $C_c^{(t)} = C(\mathbf{m}^{(t)}) - \sum_{i=1}^{N} m_i^{(t)} [\nabla_{\mathbf{m}} C(\mathbf{m}^{(t)})]_i$ can be simplified as
\begin{equation}\label{eq:constant_simplify}
\begin{aligned}
C_c^{(t)} &= -\frac{1}{4}\sum_{i=1}^{N}\sum_{j=1}^{N}(m_i^{(t)} - m_j^{(t)})^2[\mathbf{H}_{\mathbf{m}}^{+}|_{\mathbf{m}=\mathbf{1}}]_{ij} + \sum_{i=1}^{N}\sum_{j=1}^{N} m_i^{(t)}(m_i^{(t)} - m_j^{(t)})[\mathbf{H}_{\mathbf{m}}^{+}|_{\mathbf{m}=\mathbf{1}}]_{ij} \\
&= -\frac{1}{4}\sum_{i=1}^{N}\sum_{j=1}^{N}(m_i^{(t)} - m_j^{(t)})^2[\mathbf{H}_{\mathbf{m}}^{+}|_{\mathbf{m}=\mathbf{1}}]_{ij} + \frac{1}{2}\sum_{i=1}^{N}\sum_{j=1}^{N}(m_i^{(t)} - m_j^{(t)})^2[\mathbf{H}_{\mathbf{m}}^{+}|_{\mathbf{m}=\mathbf{1}}]_{ij} \\
&= \frac{1}{4}\sum_{i=1}^{N}\sum_{j=1}^{N}(m_i^{(t)} - m_j^{(t)})^2[\mathbf{H}_{\mathbf{m}}^{+}|_{\mathbf{m}=\mathbf{1}}]_{ij} \triangleq C_g^{(t)}.
\end{aligned}
\end{equation}
Here, the first equality substitutes the definition of $C(\mathbf{m}^{(t)})$ and the gradient in Eqn.~\ref{eq:concave_gradient}, and the second equality uses $\sum_{i}\sum_{j} m_i^{(t)}(m_i^{(t)} - m_j^{(t)})[\mathbf{H}_{\mathbf{m}}^{+}|_{\mathbf{m}=\mathbf{1}}]_{ij} = \frac{1}{2}\sum_{i}\sum_{j}(m_i^{(t)} - m_j^{(t)})^2[\mathbf{H}_{\mathbf{m}}^{+}|_{\mathbf{m}=\mathbf{1}}]_{ij}$, obtained by swapping $i$ and $j$ and averaging under symmetry.

Finally, substituting the consolidated linear component $L(\mathbf{m})$ from Eqn.~\ref{eq:linear_consolidated} and the bounded concave component from Eqns.~\eqref{eq:taylor_substituted}--\eqref{eq:constant_simplify} into the decomposition $g(\mathbf{m}) = L(\mathbf{m}) + C(\mathbf{m})$, and merging the coefficients of each $m_i$, we obtain the following upper bound:
\begin{equation}\label{eq:final_upper_bound}
\begin{aligned}
g(\mathbf{m}) &\le \sum_{i=1}^{N} m_i \left[ \frac{1}{2}\sum_{j=1}^N [\mathbf{H}_{\mathbf{m}}^{+}|_{\mathbf{m}=\mathbf{1}}]_{ij} - \sum_{j=1}^N [\mathbf{H}_{\mathbf{m}}^{+}|_{\mathbf{m}=\mathbf{1}}]_{ij}\left(m_i^{(t)} - m_j^{(t)}\right) \right] + C_g^{(t)} \\
&= \sum_{i=1}^{N} m_i \left[ \sum_{j=1}^N [\mathbf{H}_{\mathbf{m}}^{+}|_{\mathbf{m}=\mathbf{1}}]_{ij} \left( \frac{1}{2} - m_i^{(t)} + m_j^{(t)} \right) \right] + C_g^{(t)} \triangleq U_g(\mathbf{m}; \mathbf{m}^{(t)}),
\end{aligned}
\end{equation}
which holds for every binary $\mathbf{m}$ and coincides with Eqn.~\ref{g_m_eq_1}. Since the Taylor expansion in Eqn.~\ref{eq:expansion_talor} is exact at $\mathbf{m} = \mathbf{m}^{(t)}$, the bound is tight, i.e., $U_g(\mathbf{m}^{(t)}; \mathbf{m}^{(t)}) = g(\mathbf{m}^{(t)})$.

\subsection{Proof of Cardinality-Constrained Linear Maximization Problem} \label{sec:card_constrain}

To establish the optimization procedure at the $(t+1)$-th iteration of the MM algorithm, our objective is to maximize the tractable modular surrogate function $\mathcal{Q}(\mathbf{m}; \mathbf{m}^{(t)})$, subject to the cardinality constraint $\mathbf{1}^\top\mathbf{m} = k$. The surrogate function $\mathcal{Q}(\mathbf{m}; \mathbf{m}^{(t)})$ is constructed by subtracting the upper bound of the cost function $U_g(\mathbf{m}; \mathbf{m}^{(t)})$ from the linearization of the gain function $L_f(\mathbf{m}; \mathbf{m}^{(t)})$:
\begin{equation}
    \mathcal{Q}(\mathbf{m}; \mathbf{m}^{(t)}) = L_f(\mathbf{m}; \mathbf{m}^{(t)}) - U_g(\mathbf{m}; \mathbf{m}^{(t)}).
    \label{eq:surrogate_def_app}
\end{equation}

Recall the analytical expressions for these two functions in Eqn.~\ref{f_m_eq_1} and Eqn.~\ref{g_m_eq_1} with $\mathbf{m}' = \mathbf{m}^{(t)}$. The first-order Taylor expansion provides the linearization $L_f(\mathbf{m}; \mathbf{m}^{(t)})$ of the gain function $f(\mathbf{m})$ around the state $\mathbf{m}^{(t)}$:
\begin{equation}
    L_f(\mathbf{m}; \mathbf{m}^{(t)}) = f(\mathbf{m}^{(t)}) + \sum_{i=1}^N [\nabla_{\mathbf{m}}f(\mathbf{m}^{(t)})]_i (m_i - m_i^{(t)})
    \label{eq:lower_bound_gain},
\end{equation}
and the upper bound $U_g(\mathbf{m}; \mathbf{m}^{(t)})$ for the cost function $g(\mathbf{m})$ is given by
\begin{equation}
    U_g(\mathbf{m}; \mathbf{m}^{(t)}) = \sum_{i=1}^N m_i \left[ \sum_{j=1}^N [\mathbf{H}^+_{\mathbf{m}}|_{\mathbf{m}=\mathbf{1}}]_{ij} \left( \frac{1}{2} - m^{(t)}_i + m^{(t)}_j \right) \right] + C^{(t)}_g. 
    \label{eq:upper_bound_cost}
\end{equation}

We substitute these exact expressions from Eqn.~\ref{eq:lower_bound_gain} and Eqn.~\ref{eq:upper_bound_cost} into the surrogate function $\mathcal{Q}(\mathbf{m}; \mathbf{m}^{(t)})$ in Eqn.~\ref{eq:surrogate_def_app}, yielding 
\begin{align}
   \mathcal{Q}(\mathbf{m};\mathbf{m}^{(t)}) &= \left[ f(\mathbf{m}^{(t)}) + \sum_{i=1}^N [\nabla_{\mathbf{m}} f(\mathbf{m}^{(t)})]_i (m_i - m_i^{(t)}) \right] \nonumber \\&- \left[ \sum_{i=1}^N m_i \left( \sum_{j=1}^N [\mathbf{H}^+_{\mathbf{m}}|_{\mathbf{m}=\mathbf{1}}]_{ij} \left( \frac{1}{2} - m^{(t)}_i + m^{(t)}_j \right) \right) + C^{(t)}_g \right]. 
    \label{eq:surrogate_subbed}
\end{align}

After that, we expand the inner product associated with the gradient term in the linearization $L_f(\mathbf{m}; \mathbf{m}^{(t)})$, and distribute the gradient $[\nabla_{\mathbf{m}}f(\mathbf{m}^{(t)})]_i$ across the difference $m_i - m_i^{(t)}$:
\begin{equation}
    \sum_{i=1}^N [\nabla_{\mathbf{m}}f(\mathbf{m}^{(t)})]_i (m_i - m_i^{(t)}) = \sum_{i=1}^N [\nabla_{\mathbf{m}}f(\mathbf{m}^{(t)})]_i m_i - \sum_{i=1}^N [\nabla_{\mathbf{m}}f(\mathbf{m}^{(t)})]_i m_i^{(t)}.
    \label{eq:gradient_expansion}
\end{equation}

Substituting the expanded form from Eqn.~\ref{eq:gradient_expansion} into our surrogate function in Eqn.~\ref{eq:surrogate_subbed} yields:
\begin{align}
  \mathcal{Q}(\mathbf{m};\mathbf{m}^{(t)}) &= f(\mathbf{m}^{(t)}) + \sum_{i=1}^N [\nabla_{\mathbf{m}} f(\mathbf{m}^{(t)})]_i m_i - \sum_{i=1}^N [\nabla_{\mathbf{m}} f(\mathbf{m}^{(t)})]_i m_i^{(t)} \nonumber \\&- \sum_{i=1}^N m_i \left( \sum_{j=1}^N [\mathbf{H}^+_{\mathbf{m}}|_{\mathbf{m}=\mathbf{1}}]_{ij} \left( \frac{1}{2} - m^{(t)}_i + m^{(t)}_j \right) \right) - C^{(t)}_g.
    \label{eq:surrogate_expanded}
\end{align}

Next, we regroup the terms in Eqn.~\ref{eq:surrogate_expanded} to segregate those dependent on the decision variable $\bf m$ from those that are constant at iteration $t$:
\begin{align}
    \mathcal{Q}(\mathbf{m};\mathbf{m}^{(t)}) &= \underbrace{ \sum_{i=1}^N [\nabla_{\mathbf{m}} f(\mathbf{m}^{(t)})]_i m_i - \sum_{i=1}^N m_i \left( \sum_{j=1}^N [\mathbf{H}^+_{\mathbf{m}}|_{\mathbf{m}=\mathbf{1}}]_{ij} \left( \frac{1}{2} - m^{(t)}_i + m^{(t)}_j \right) \right) }_{\text{Terms dependent on } \mathbf{m}} \nonumber \\&+ \underbrace{ f(\mathbf{m}^{(t)}) - \sum_{i=1}^N [\nabla_{\mathbf{m}} f(\mathbf{m}^{(t)})]_i m_i^{(t)} - C^{(t)}_g }_{\text{Constant terms independent of } \mathbf{m}}. 
    \label{eq:surrogate_grouped}
\end{align}

Let $C^{(t)}$ incorporate all terms from Eqn.~\ref{eq:surrogate_grouped} that are independent of the decision variable $\mathbf{m}$:
\begin{equation}
    C^{(t)} = f(\mathbf{m}^{(t)}) - \sum_{i=1}^N [\nabla_{\mathbf{m}} f(\mathbf{m}^{(t)})]_i m_i^{(t)} - C^{(t)}_g.
    \label{eq:constant_c}
\end{equation}

Substituting the constant $C^{(t)}$ from Eqn.~\ref{eq:constant_c} into Eqn.~\ref{eq:surrogate_grouped} simplifies the formulation:
\begin{equation}
    \mathcal{Q}(\mathbf{m};\mathbf{m}^{(t)}) = \sum_{i=1}^N [\nabla_{\mathbf{m}} f(\mathbf{m}^{(t)})]_i m_i - \sum_{i=1}^N m_i \left( \sum_{j=1}^N [\mathbf{H}^+_{\mathbf{m}}|_{\mathbf{m}=\mathbf{1}}]_{ij} \left( \frac{1}{2} - m^{(t)}_i + m^{(t)}_j \right) \right) + C^{(t)}. 
    \label{eq:surrogate_simplified_c}
\end{equation}

Factoring out the binary variable $m_i$ from the remaining summations in Eqn.~\ref{eq:surrogate_simplified_c} isolates the marginal contribution of each expert $E_i$: 
\begin{equation}
    \mathcal{Q}(\mathbf{m};\mathbf{m}^{(t)}) = \sum_{i=1}^N m_i \left( [\nabla_{\mathbf{m}} f(\mathbf{m}^{(t)})]_i - \sum_{j=1}^N [\mathbf{H}^+_{\mathbf{m}}|_{\mathbf{m}=\mathbf{1}}]_{ij} \left( \frac{1}{2} - m^{(t)}_i + m^{(t)}_j \right) \right) + C^{(t)}.
    \label{eq:surrogate_factored}
\end{equation}

We define the importance score $S_i^{(t)}$ for the $i$-th expert, which quantifies the trade-off between the expert's synergy gain and its redundancy penalty:
\begin{equation}
    S_i^{(t)} = [\nabla_{\mathbf{m}} f(\mathbf{m}^{(t)})]_i - \sum_{j=1}^N [\mathbf{H}^+_{\mathbf{m}}|_{\mathbf{m}=\mathbf{1}}]_{ij} \left( \frac{1}{2} - m^{(t)}_i + m^{(t)}_j \right).
    \label{eq:score_def}
\end{equation}

By substituting $S_i^{(t)}$ from Eqn.~\ref{eq:score_def} into Eqn.~\ref{eq:surrogate_factored}, the complex surrogate function is reduced to an affine form:
\begin{equation}
    \mathcal{Q}(\mathbf{m}; \mathbf{m}^{(t)}) = \sum_{i=1}^N m_i S_i^{(t)} + C^{(t)}.
    \label{eq:surrogate_affine}
\end{equation}

The iterative update step of the MM algorithm requires finding $\mathbf{m}^{(t+1)} = \arg\max_{\mathbf{m}} \mathcal{Q}(\mathbf{m}; \mathbf{m}^{(t)})$, subject to the cardinality constraint $\mathbf{1}^\top\mathbf{m} = k$. Since $C^{(t)}$ acts as a uniform scalar shift across the objective landscape, it does not alter the location of the maximum. Thus, maximizing $\mathcal{Q}(\mathbf{m}; \mathbf{m}^{(t)})$ in Eqn.~\ref{eq:surrogate_affine} is equivalent to maximizing the linear summation alone:
\begin{equation}
    \mathop{\arg\max}_{\mathbf{m}\in\{0,1\}^N:\,\mathbf{1}^\top\mathbf{m} = k} \left( \sum_{i=1}^N m_i S_i^{(t)} + C^{(t)} \right) = \mathop{\arg\max}_{\mathbf{m}\in\{0,1\}^N:\,\mathbf{1}^\top\mathbf{m} = k} \sum_{i=1}^N m_i S_i^{(t)}.
    \label{eq:argmax_equivalence}
\end{equation}
This yields the update rule in Eqn.~\ref{eq:mm_update}, whose exact solution is obtained by selecting the $k$ experts with the largest scores $S_i^{(t)}$ via Top-$k$ sorting.

\subsection{Proof of Convergence for the MM Algorithm} \label{sec:mm_convergence}

To prove the convergence property of the MM algorithm presented in Theorem~\ref{thm:convergence}, we demonstrate that the sequence of the retained utility values $\{\mathcal{U}(\mathbf{m}^{(t)})\}_{t \ge 1}$ generated by the iterative updates is monotonically non-decreasing and strictly increases whenever the selection changes. The MM algorithm operates by sequentially maximizing the surrogate function $\mathcal{Q}(\mathbf{m}; \mathbf{m}^{(t)})$ constructed in Eqn.~\ref{eqn:surro_func}, which satisfies the following two properties.

First, the surrogate coincides with the true objective at the current anchor state $\mathbf{m} = \mathbf{m}^{(t)}$. For the linearization in Eqn.~\ref{f_m_eq_1}, the deviation terms $(m_i - m_i^{(t)})$ vanish at the anchor state, so that $L_f(\mathbf{m}^{(t)};\mathbf{m}^{(t)}) = f(\mathbf{m}^{(t)})$. For the upper bound in Eqn.~\ref{g_m_eq_1}, the tangent expansion of the concave component is exact at the anchor state (Appendix~\ref{sec:cost_upper}), so that $U_g(\mathbf{m}^{(t)}; \mathbf{m}^{(t)}) = g(\mathbf{m}^{(t)})$. Subtracting these yields the tightness condition:
\begin{equation}
    \mathcal{Q}(\mathbf{m}^{(t)}; \mathbf{m}^{(t)}) = \mathcal{U}(\mathbf{m}^{(t)}).
    \label{eq:surrogate_tightness}
\end{equation}

Second, during the update phase defined by the cardinality-constrained linear maximization in Eqn.~\ref{eq:mm_update}, the algorithm selects the candidate $\mathbf{m}^{(t+1)}$ as the exact maximizer of the surrogate function over the feasible region, which contains $\mathbf{m}^{(t)}$ for $t \ge 1$. By the definition of a maximizer, the surrogate value cannot degrade:
\begin{equation}
    \mathcal{Q}(\mathbf{m}^{(t+1)}; \mathbf{m}^{(t)}) \ge \mathcal{Q}(\mathbf{m}^{(t)}; \mathbf{m}^{(t)}).
    \label{eq:surrogate_improvement}
\end{equation}

\begin{algorithm}[t!]
    \caption{MM-based Expert Selection}
    \label{alg:mm}
    \footnotesize
    \Input{Hessian matrix $\mathbf{H}_{\bf m}|_{\bf m=1} \in \mathbb{R}^{N \times N}$, utility vector $\mathbf{d} \in \mathbb{R}^N$, sparsity $k$, max iterations $T_{\max}$}
    \Output{Selected expert subset $\mathbf{m}^* \in \{0,1\}^N$ s.t. $\mathbf{1}^\top\mathbf{m}^* = k$}
    \BlankLine
    // \texttt{Step 1: Initialization}\\
    $t \gets 0, \quad \mathbf{m}^{(0)} \gets \mathbf{0}$ \Comment*[r]{Start with empty selection}
    
    \BlankLine
    // \texttt{Step 2: Hessian Decomposition}\\
    $[\mathbf{H}_{\bf m}^+]_{ij} \gets \max(0, [\mathbf{H}_{\bf m}|_{\bf m=1}]_{ij})$ \Comment*[r]{Redundancy component}
    $[\mathbf{H}_{\bf m}^-]_{ij} \gets \max(0, -[\mathbf{H}_{\bf m}|_{\bf m=1}]_{ij})$ \Comment*[r]{Synergy component}
    
    \BlankLine
    // \texttt{Step 3: MM Optimization}\\
    \While{$t < T_{\max}$}{
        \tcp*[l]{Construct Surrogate Function}
        \For{$i = 1$ \KwTo $N$}{
            $[\nabla_{\bf m} f(\mathbf{m}^{(t)})]_i \gets d_i + \sum_{j=1}^N [\mathbf{H}_{\bf m}^-]_{ij} m_j^{(t)}$ \Comment*[r]{Marginal synergy gain}
            $p_i^{(t)} \gets \sum_{j=1}^N [\mathbf{H}_{\bf m}^+]_{ij} \left(\frac{1}{2} - m_i^{(t)} + m_j^{(t)}\right)$ \Comment*[r]{Context-aware redundancy penalty}
            $S_i^{(t)} \gets [\nabla_{\bf m} f(\mathbf{m}^{(t)})]_i - p_i^{(t)}$ \Comment*[r]{Importance score}
        }
        
        \BlankLine
        \tcp*[l]{Update Selection via Top-$k$ Projection}
        $\mathbf{m}^{(t+1)} \gets \text{Top-$k$}(\mathbf{S}^{(t)}, k)$\;
        
        \BlankLine
        \If{$t \ge 1$ \textbf{and} $\mathcal{U}(\mathbf{m}^{(t+1)}) \le \mathcal{U}(\mathbf{m}^{(t)})$}{
            \textbf{break} \Comment*[r]{Reject non-improving update}
        }
        \If{$\mathbf{m}^{(t+1)} = \mathbf{m}^{(t)}$}{
            \textbf{break} \Comment*[r]{Selection converged}
        }
        $t \gets t + 1$\;
    }
    \Return $\mathbf{m}^* \gets \mathbf{m}^{(t)}$\;
\end{algorithm}

However, since $\mathcal{Q}(\mathbf{m}; \mathbf{m}^{(t)})$ does not necessarily lower-bound $\mathcal{U}(\mathbf{m})$, Eqns.~\eqref{eq:surrogate_tightness}--\eqref{eq:surrogate_improvement} only imply $\mathcal{Q}(\mathbf{m}^{(t+1)}; \mathbf{m}^{(t)}) \ge \mathcal{U}(\mathbf{m}^{(t)})$ and do not by themselves guarantee $\mathcal{U}(\mathbf{m}^{(t+1)}) \ge \mathcal{U}(\mathbf{m}^{(t)})$. The acceptance rule in Algorithm~\ref{alg:mm} closes this gap: for every $t \ge 1$, the candidate $\mathbf{m}^{(t+1)}$ is retained only if $\mathcal{U}(\mathbf{m}^{(t+1)}) > \mathcal{U}(\mathbf{m}^{(t)})$; otherwise, the algorithm terminates and the sequence remains at $\mathbf{m}^{(t)}$. Consequently, the sequence of retained utilities satisfies
\begin{equation}
    \mathcal{U}(\mathbf{m}^{(t+1)}) \ge \mathcal{U}(\mathbf{m}^{(t)}), \quad \forall t \ge 1,
    \label{eq:monotonicity_final}
\end{equation}
where the inequality is strict whenever $\mathbf{m}^{(t+1)} \neq \mathbf{m}^{(t)}$, which establishes Eqn.~\ref{eq:U_m_inequlity}. Because $\mathcal{U}$ strictly increases along every change of selection, no feasible selection can be visited twice. Since the feasible set $\{\mathbf{m} \in \{0,1\}^N : \mathbf{1}^\top\mathbf{m} = k\}$ contains $\binom{N}{k}$ elements, the algorithm terminates after at most $\binom{N}{k}$ changes of selection, which completes the proof.

\subsection{The Complete Majorization-Minimization Algorithm} \label{sec:mm_algo}

Algorithm~\ref{alg:mm} outlines the complete procedure of the proposed MM optimization for expert selection. The algorithm initializes the iteration counter $t=0$ and sets the starting selection vector $\mathbf{m}^{(0)}$ to an empty state $\mathbf{0}$ ({\texttt{Line 2}}). Next, the Hessian matrix $\mathbf{H}_{\mathbf{m}}|_{\mathbf{m}=\mathbf{1}}$ is decomposed into non-negative functional redundancy component $\mathbf{H}_{\mathbf{m}}^+$ and complementary synergy component $\mathbf{H}_{\mathbf{m}}^-$ ({\texttt{Lines 4-5}}). Inside the MM optimization loop, the algorithm dynamically calculates the marginal synergy gain $[\nabla_{\mathbf{m}} f(\mathbf{m}^{(t)})]_i$ for each candidate expert relative to the active subset $\mathbf{m}^{(t)}$ selected in the preceding iteration ({\texttt{Line 9}}). Simultaneously, it computes a context-aware redundancy penalty $p_i^{(t)}$ based on the current selection state to evaluate functional overlap ({\texttt{Line 10}}). By subtracting the redundancy penalty from the marginal synergy gain, the algorithm formulates the composite importance score ${S}_i^{(t)}$ ({\texttt{Line 11}}). The candidate selection vector $\mathbf{m}^{(t+1)}$ is then determined via a Top-$k$ projection of these evaluated scores ({\texttt{Line 12}}). For $t \ge 1$, the candidate is accepted only if it strictly increases the retained utility $\mathcal{U}(\mathbf{m})$, which is evaluated exactly with $\mathcal{O}(N^2)$ cost; otherwise, the loop terminates ({\texttt{Lines 13-14}}). Finally, the algorithm enforces a termination condition: if the newly selected subset precisely matches the previous one ($\mathbf{m}^{(t+1)} = \mathbf{m}^{(t)}$), the optimization has reached a stable state, prompting the loop to break and return the selected expert subset $\mathbf{m}^*$ ({\texttt{Lines 15-18}}).

\subsection{Datasets and Evaluation Metrics} \label{sec:datasets}
% [cite: 350, 351, 585] 补充 AGNews, PIQA, HellaSwag, MMLU 的详细描述。

\textbf{Datasets.} To evaluate the proposed DS-MoE framework against the state-of-the-art routing baselines, we conduct comprehensive empirical evaluations across four widely adopted datasets spanning a spectrum from fundamental semantic comprehension to advanced, knowledge-intensive reasoning. The selected benchmarks evaluate the MoE architectures across varying degrees of complexity and domain heterogeneity:
\begin{itemize}
    \item \textbf{AGNews (Natural Language Understanding)}~\cite{zhang2015character}: A fundamental benchmark for text classification comprising news articles categorized into four topical domains (World, Sports, Business, Sci/Tech). We utilize this dataset to evaluate the model's capacity to capture linearly separable semantic representations,  verifying whether the routing mechanism can distinctly allocate experts to non-overlapping topics without representation collapse.
    \item \textbf{PIQA (Physical Interaction Question Answering)}~\cite{bisk2020piqa}: A physical commonsense reasoning benchmark requiring the deduction of the most plausible continuation of a scenario governed by real-world physical laws. This task challenges the MoE framework to dynamically activate expert subsets specialized in intuitive physics and spatial-temporal logic, extending significantly beyond static factoid retrieval.
    \item \textbf{HellaSwag (Predictive Reasoning)}~\cite{zellers2019hellaswag}: A challenging dataset for evaluating grounded, adversarial natural language reasoning. Because the incorrect options in HellaSwag are adversarially generated by machines to be structurally coherent but logically flawed, this benchmark explicitly stresses the MoE's ability to disentangle spurious correlations, explicitly highlighting the necessity of synergistic expert combinations (error cancellation) to effectively resolve complex logical reasoning.
    \item \textbf{MMLU (Massive Multitask Language Understanding)}~\cite{hendrycksmeasuring}: A comprehensive benchmark designed to evaluate zero-shot and few-shot knowledge-intensive reasoning across 57 distinct academic and professional disciplines (e.g., STEM, humanities, social sciences). Given its extreme domain heterogeneity, MMLU is critical for stress-testing the MoE architecture. It validates the MoE framework's capacity to retain a diverse expert pool and orchestrate combinatorial expert collaboration, while effectively mitigating catastrophic forgetting and negative transfer.  
\end{itemize}

\textbf{Evaluation Metrics.} To capture both the final generalization capability and the underlying optimization dynamics, we assess the models using two primary quantitative metrics:
\begin{itemize}
    \item \textbf{Test Accuracy (\%):} The standard Top-1 accuracy measured on the unseen test splits across all evaluated tasks. As the primary proxy for generalization efficacy, higher test accuracy provides direct empirical validation that the proposed expert selection successfully preserves indispensable, synergistic expert representations while discarding redundant ones.
    \item \textbf{Convergence Efficiency (Rounds):} Measured by the number of training iterations required for the validation performance to reach a stationary accuracy threshold. Standard Top-$k$ routing often suffers from slow convergence or oscillatory behavior due to conflicting gradient updates from functionally redundant experts. By tracking the optimization trajectory, this metric quantitatively demonstrates how the proposed framework's mitigation of redundancy translates to accelerated gradient flow and superior training stability.
\end{itemize}

\subsection{Model Architectures and Hyperparameter Settings} \label{sec:hyperparams}

To conduct the evaluation of the DS-MoE framework, we benchmark our approach against three representative MoE architectures: Switch-Base-32, Switch-Base-64, and DeepSeek-MoE-16B. The detailed model architecture and experimental configurations are outlined below.

\paragraph{Model Architectures}
\begin{itemize}
    \item \textbf{Switch-Base Variants ($N \in \{32, 64\}$):} Both Switch-Base-32 and Switch-Base-64 models adopt a symmetric sequence-to-sequence architecture, comprising 12 encoder layers and 12 decoder layers. We uniformly set the hidden dimension to $d_{\text{model}} = 768$ and the vocabulary size to 32,128. These variants employ an alternating arrangement of dense and sparse layers: specifically, half of the Transformer blocks utilize standard dense feed-forward networks (FFN) with an expansion dimension of 3,072, while the remaining blocks are replaced by sparse MoE modules integrating 32 (resp. Switch-Base-32) and 64 (resp. Switch-Base-64) experts. To ensure computational equivalence for a single forward pass, the internal expansion dimension of all individual experts is strictly locked at 3,072, ensuring parity with the standard dense FFN layers.
    \item \textbf{DeepSeek-MoE-16B ($N = 64$):} We instantiate the DeepSeek-MoE-16B as a 28-layer decoder-only autoregressive architecture. We set the hidden dimension $d_{\text{model}} = 2,048$ and the vocabulary size to 102,400. To stabilize shallow-layer feature representations, the foundational layer incorporates a standard dense FFN with an intermediate dimension of 10,944. The subsequent 27 layers are configured as MoE modules. Within each MoE layer, we deploy a hybrid expert configuration consisting of 64 independent routed experts with an individual dimension of 1,408.   
\end{itemize}

\paragraph{Experimental Configurations}
\begin{itemize}
    \item \textbf{Optimization and Parameter Updating Strategy:} All experiments are executed using bfloat16 mixed precision and the 8-bit AdamW optimizer. The learning rate is uniformly set to $1 \times 10^{-4}$ across all tasks. We configure a micro-batch size of 4 coupled with 4 gradient accumulation steps, yielding a batch size of 16. During fine-tuning, only the sparse components (expert networks and routers) and task-specific heads are updated, while the dense backbone remains frozen.
    \item \textbf{Task-Specific Setup:} We adapt the maximum sequence length based on the inherent context requirements of each specific benchmark. Specifically, the lengths are configured to 128 tokens for AGNews; 256 tokens for PIQA; and 512 tokens for HellaSwag and MMLU. During the evaluation of generative multiple-choice tasks (such as HellaSwag and MMLU), the generation configuration is strictly bounded to a maximum of 4 new tokens.
\end{itemize}

{\bf{Specific Configurations for MM Algorithm}.}
The proposed MM expert selection framework introduces several algorithm-specific hyperparameters. We set the maximum iteration threshold of $T_{\text{max}} = 10$ and active expert capacity $k = 2$ by default unless
specified otherwise.

{\bf{Hardware.}} All experiments are conducted on a high-performance central server equipped with NVIDIA GeForce RTX 4090 GPUs.

%%%%%-----------------------------------------------------------

%%%%%-----------------------------------------------------------
% --- D. Limitations and Future Work ---
\section{Limitations and Future Work} \label{sec:limitations}

While the DS-MoE framework establishes a rigorous mathematical foundation for optimizing expert combinations, our current analysis confines the modeling of combinatorial dependencies to intra-layer interactions. Given the sequential feature transformations inherent in deep language models, capturing inter-layer dynamics presents a critical avenue for exploration. A promising direction for future research involves investigating cross-layer expert synergies, such as approximating a global, block-diagonal Hessian proxy, to enable holistic, network-wide routing strategies that further maximize overall parameter efficiency. Moreover, the proposed framework can be naturally extended to distributed machine learning scenarios, such as federated learning~\cite{zhang2024satfed,lin2024fedsn,teng2026communication,fang2024automated,hu2024accelerating,qiu2024ifvit} or split learning~\cite{lin2024adaptsfl,zhang2025s,wei2025optimizing,lin2025hasfl,lyu2023optimal,fang2026nsc,lin2026gapsl,lyu2023optimal}, where expert selection can be jointly optimized with communication and computation constraints to improve both parameter efficiency and system efficiency.

\section{Practical Impact} \label{sec:practical_impact}

By bridging rigorous theoretical guarantees with expert selection optimizations, the DS-MoE framework addresses critical bottlenecks in the training and deployment of large-scale MoE architectures:
\begin{itemize}
    \item \textbf{Theoretically Grounded Routing Paradigm:} The training of industrial-scale MoE models is impeded by routing collapse and optimization instabilities, necessitating extensive hyperparameter tuning. By introducing an MM framework with guaranteed monotonic convergence, DS-MoE circumvents the heuristic unpredictability inherent in standard Top-$k$ routing, thereby enhancing the model performance. 
    \item \textbf{Facilitating Deployment in Resource-Constrained Edge Systems:} By optimizing representational capacity under strict cardinality constraints (e.g., a predefined expert budget $k$), DS-MoE drastically reduces the memory footprint and computing requirements for training and deploying LLMs. This enables resource-constrained computing entities to leverage the state-of-the-art MoE architectures without compromising performance.
    \item \textbf{Sustainable and Cost-Effective Training:} As demonstrated in our empirical results, decoupling functional redundancy from expert synergy resolves gradient conflicts and substantially expedites model convergence. This algorithmic efficiency translates directly to significant reductions in computing overhead and overall energy consumption, rendering the training of large-scale MoE models highly resource-efficient and economically viable.
\end{itemize}

\end{document}